\documentclass[oneside,onecolumn]{article}
\usepackage{microtype} 

\usepackage[english]{babel} 

\usepackage[hmarginratio=1:1,top=32mm,columnsep=20pt]{geometry} 
\usepackage{booktabs} 

\usepackage{lettrine} 

\usepackage{url,multirow,comment}
\usepackage{amsmath,amssymb,amsthm}
\usepackage{graphicx,color}
\usepackage[figuresright]{rotating}
\usepackage[utf8]{inputenc} 
\usepackage[T1]{fontenc}    
\usepackage{url}            
\usepackage{booktabs}       
\usepackage{enumitem}
\usepackage{amsfonts}       
\usepackage{amssymb,mathtools}
\usepackage{amsfonts}
\usepackage{nicefrac}       
\usepackage{microtype}      
\usepackage{amsmath,amssymb,amsthm}

\usepackage{subcaption}
\usepackage{pbox}
\usepackage{tkz-euclide}
\usetkzobj{all}
\usepackage{tikz}
\usepackage{url}
\usepackage{todonotes}
\usepackage[bookmarks=false,colorlinks=true,urlcolor=blue,citecolor=blue,linkcolor=blue]{hyperref}
\usepackage{cleveref}
\usepackage{bm}
\usepackage{cite}
\allowdisplaybreaks
\renewcommand\footnotemark{}

\newtheorem{lem}{Lemma}
\newtheorem{thm}{Theorem}
\newtheorem{defn}{Definition}
\newtheorem{coro}{Corollary}

\newtheorem{rem}{Remark}

\newcommand{\E}[1]{\mathbb{E}\left[{#1}\right]}

\newcommand{\Esub}[2]{\mathbb{E}_{#1}\left[{#2}\right]}
\newcommand{\wts}{\mathbf{w}}
\newcommand{\mb}{m}
\newcommand{\iters}{J}
\newcommand{\learners}{P}
\newcommand{\lips}{L}

\crefname{equation}{}{}
\Crefname{equation}{}{}
\crefname{thm}{theorem}{theorems}
\Crefname{thm}{Theorem}{Theorems}
\crefname{clm}{claim}{claims}
\Crefname{clm}{Claim}{Claims}
\Crefname{coro}{Corollary}{Corollaries}
\Crefname{lem}{Lemma}{Lemmas}
\Crefname{sec}{Section}{Sections}
\crefname{app}{appendix}{appendices}
\Crefname{app}{Appendix}{Appendices}
\Crefname{part}{Part}{Parts}
\crefname{prop}{proposition}{propositions}
\Crefname{prop}{Proposition}{Propositions}
\Crefname{propty}{Property}{Properties}
\crefname{figure}{figure}{figures}
\Crefname{figure}{Figure}{Figures}
\crefname{defn}{definition}{definitions}
\Crefname{defn}{Definition}{Definitions}
\crefname{fact}{fact}{facts}
\Crefname{fact}{Fact}{Facts}
\crefname{appendix}{appendix}{appendices}
\Crefname{appendix}{Appendix}{Appendices}
\crefname{algo}{algorithm}{algorithms}
\Crefname{algo}{Algorithm}{Algorithms}
\crefname{algorithm}{algorithm}{algorithms}
\Crefname{algorithm}{Algorithm}{Algorithms}
\crefname{conj}{conjecture}{conjectures}
\Crefname{conj}{Conjecture}{Conjectures}
\crefname{obs}{observation}{observations}
\Crefname{obs}{Observation}{Observations}
\crefname{rem}{remark}{remarks}
\Crefname{rem}{Remark}{Remarks}

\begin{document}
%

%


\title{Slow and Stale Gradients Can Win the Race: Error-Runtime Trade-offs in Distributed SGD}

\author{Sanghamitra Dutta$^1$, Gauri Joshi$^1$,  Soumyadip Ghosh$^2$. \\ Parijat Dube$^2$, and Priya Nagpurkar$^2$ \\
\thanks{S. Dutta and G. Joshi are with the Department of Electrical and Computer Engineering, Carnegie Mellon University. S. Ghosh, P. Dube and P. Nagpurkar are with IBM Research. This work was done when G. Joshi was a research staff member at IBM Research and S. Dutta was an intern. Author Contacts: S. Dutta (sanghamd@andrew.cmu.edu), G. Joshi (gaurij@andrew.cmu.edu), S. Ghosh (ghoshs@us.ibm.com), P. Dube (pdube@us.ibm.com) and P. Nagpurkar (pnagpurkar@us.ibm.com).}
\thanks{Presented at the International Conference on Artificial Intelligence and Statistics (AISTATS) 2018, Lanzarote, Spain.}
\normalsize $^1$Carnegie Mellon University,   $^2$IBM Research
}







\maketitle
\begin{abstract}

Distributed Stochastic Gradient Descent (SGD) when run in a synchronous manner, suffers from delays in waiting for the slowest learners (stragglers). Asynchronous methods can alleviate stragglers, but cause gradient staleness that can adversely affect convergence. In this work we present a novel theoretical characterization of the speed-up offered by asynchronous methods by analyzing the trade-off between the error in the trained model and the actual training runtime (wallclock time). The novelty in our work is that our runtime analysis considers random straggler delays, which helps us design and compare distributed SGD algorithms that strike a balance between stragglers and staleness. We also present a new convergence analysis of asynchronous SGD variants without bounded or exponential delay assumptions, and a novel learning rate schedule to compensate for gradient staleness.
\end{abstract}
%

\section{INTRODUCTION}
%
%
Stochastic gradient descent (SGD) is the backbone of most state-of-the-art machine learning algorithms. Thus, improving the stability and convergence rate of SGD algorithms is critical for making machine learning algorithms fast and efficient.


Traditionally SGD is run serially at a single node. However, for massive datasets, running SGD serially at a single server can be \textit{prohibitively} slow. A solution that has proved successful in recent years is to parallelize the training across many learners (processing units). This method was first used at a large-scale in Google's DistBelief \cite{dean2012large} which used a central parameter server (PS) to aggregate gradients computed by learner nodes. While parallelism dramatically speeds up training, distributed machine learning frameworks face several challenges such as: 

\textbf{Straggling Learners.} In synchronous SGD, the PS waits for all learners to push gradients before it updates the model parameters. Random delays in computation (referred to as straggling) are common in today's distributed systems \cite{dean2013tail}. Waiting for slow and straggling learners can diminish the speed-up offered by parallelizing the training.


\textbf{Gradient Staleness.} To alleviate the problem of stragglers, SGD can be run in an asynchronous manner, where the central parameters are updated without waiting for all learners. However, learners may return \emph{stale} gradients that were evaluated at an older version of the model, and this can make the algorithm unstable. 

The key contributions of this work are: 
\begin{enumerate}[leftmargin=*]
\item Most SGD algorithms optimize the trade-off between training error, and the number of iterations or epochs. However, the wallclock time per iteration is a random variable that depends on the gradient aggregation algorithm. We present a rigorous analysis of the trade-off between error and the actual runtime (instead of iterations), \textit{modelling runtimes as random variables with a general distribution.} This analysis is then used to compare different SGD variants such as $K$-sync SGD, $K$-async SGD and $K$-batch-async SGD, as illustrated in \Cref{fig:error runtime tradeoff}.
\item We present a new convergence analysis of asynchronous SGD and some of its variants, where we relax several commonly made assumptions such as bounded delays and gradients, exponential service times, and independence of the staleness process.
 \item We propose a novel learning rate schedule to compensate for gradient staleness, and improve the stability and convergence of asynchronous SGD, while preserving its fast runtime. 
\end{enumerate}

 \begin{figure}[t]
 \centering
\includegraphics[height=5cm]{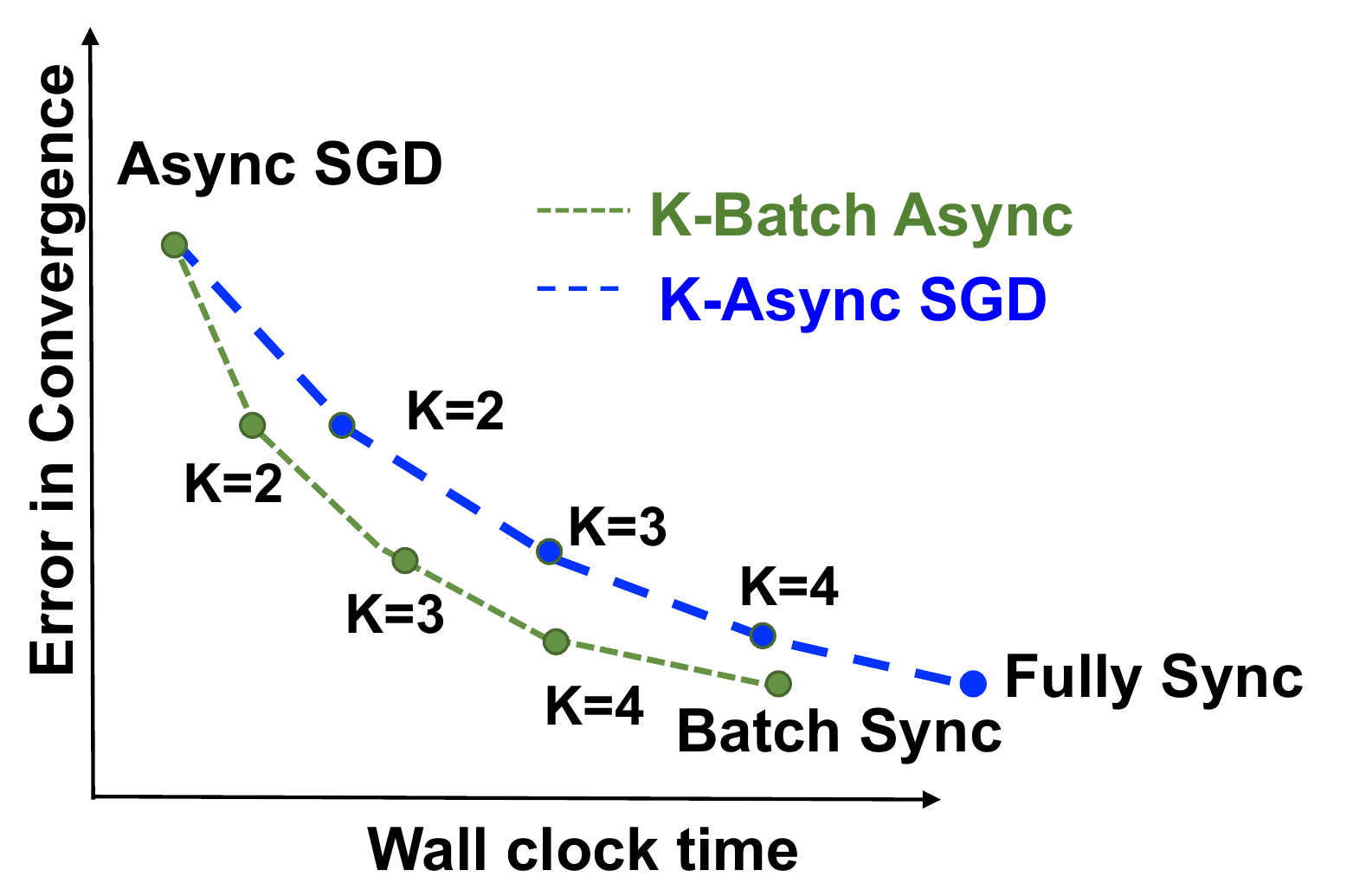}
\caption{SGD variants span the error-runtime trade-off between fully Sync-SGD and fully Async-SGD. $K$ is the number of learners or mini-batches the PS waits for before updating the model parameters, as we elaborate in \Cref{sec:system model}.}
\label{fig:error runtime tradeoff}
\end{figure}

\subsection{RELATED WORKS}

\textbf{Single Node SGD:}
Analysis of gradient descent dates back to classical works \cite{boyd2004convex} in the optimization community. The problem of interest is the minimization of empirical risk of the form:
\begin{equation}
\min_{\wts} \left\{ F(\wts)\overset{\text{def}}{=} \frac{1}{N}\sum_{n=1}^{N} f(\wts, \xi_n) \right\}.
\label{eq:minimization}
\end{equation}
Here, $\xi_n$ denotes the $n-$th data point and its label where $n=1,2,\dots,N$, and $f(\wts, \xi_n)$ denotes the composite loss function. Gradient descent is a way to iteratively minimize this objective function by updating the parameter $\wts$ in the opposite direction of the gradient of $F(\wts)$ at every iteration, as given by: $$\wts_{j+1}=\wts_{j}-\eta \nabla F(\wts_{j}) = \wts_{j} - \frac{\eta}{N} \sum_{n=1}^N \nabla f(\wts_{j}, \xi_n).$$
The computation of $\sum_{n=1}^N \nabla f(\wts_{j}, \xi_n)$ over the entire dataset is expensive. Thus, stochastic gradient descent \cite{robbins1951stochastic} with mini-batching is generally used in practice, where the gradient is evaluated over small, randomly chosen subsets of the data. Smaller mini-batches result in higher variance of the gradients, which affects convergence and error floor \cite{dekel2012optimal, li2014efficient, bottou2016optimization}. Algorithms such as AdaGrad \cite{duchi2011adaptive} and Adam \cite{kingma2015adam} gradually reduce learning rate to achieve a lower error floor. Another class of algorithms includes stochastic variation reduction techniques that include SVRG \cite{johnson2013accelerating}, SAGA \cite{roux2012stochastic} and their variants listed out in \cite{nguyen2017sarah}. For a detailed survey of different SGD variants, refer to \cite{ruder2016overview}.



\textbf{Synchronous SGD and Stragglers:} 
To process large datasets, SGD is parallelized across multiple learners with a central PS. Each learner processes one mini-batch, and the PS aggregates all the gradients. The convergence of synchronous SGD is same as mini-batch SGD, with a $P$-fold larger mini-batch, where $P$ is the number of learners. However, the time per iteration grows with the number of learners, because some straggling learners that slow down randomly \cite{dean2013tail}. Thus, it is important to juxtapose the error reduction per iteration with the runtime per iteration to understand the true convergence speed of distributed SGD. 

To deal with stragglers and speed up machine learning, system designers have proposed several straggler mitigation techniques such as \cite{harlap2016addressing} that try to detect and avoid stragglers. An alternate direction of work is to use redundancy techniques, e.g., replication or erasure codes, as proposed in \cite{joshi2014delay,wang2015using,joshi2015queues, joshi2017efficient, lee2017speeding, tandon2017gradient,dutta2016short,halbawi2017improving,yang2017coded,yang2016fault,yu2017polynomial,karakus2017encoded,karakus2017straggler,charles2017approximate,li2017terasort,fahim2017optimal,ye2018communication,li2018fundamental,NewsletterPaper,DNNPaperISIT,mallick2018rateless} to deal with the stragglers, as also discussed in \Cref{rem:redundancy_techniques}.

\textbf{Asynchronous SGD and Staleness:}
A complementary approach to deal with the issue of straggling is to use asynchronous SGD. In asynchronous SGD, any learner can evaluate the gradient and update the central PS without waiting for the other learners. Asynchronous variants of existing SGD algorithms have also been proposed and implemented in systems \cite{dean2012large, gupta2016model,cipar2013solving, cui2014exploiting,ho2013more}.

In general, analyzing the convergence of asynchronous SGD with the number of iterations is difficult in itself because of the randomness of gradient staleness. There are only a few pioneering works such as \cite{tsitsiklis1986distributed,lian2015asynchronous,mitliagkas2016asynchrony,recht2011hogwild,agarwal2011distributed, mania2017perturbed,chaturapruek2015asynchronous,zhang2016staleness,peng2016arock, hannah2017more, hannah2016unbounded,sun2017asynchronous,leblond2017asaga} in this direction. In \cite{tsitsiklis1986distributed}, a fully decentralized analysis was proposed that considers no central PS. In \cite{recht2011hogwild}, a new asynchronous algorithm called Hogwild was proposed and analyzed under bounded gradient and bounded delay assumptions. This direction of research has been followed upon by several interesting works such as \cite{lian2015asynchronous} which proposed novel theoretical analysis under bounded delay assumption for other asynchronous SGD variants. In \cite{peng2016arock, hannah2017more, hannah2016unbounded,sun2017asynchronous}, the framework of ARock was proposed for parallel co-ordinate descent and analyzed using Lyapunov functions, relaxing several existing assumptions such as bounded delay assumption and the independence of the delays and the index of the blocks being updated. In algorithms such as Hogwild, ARock etc. every learner only updates a part of the central parameter vector $\wts$ at every iteration and are thus essentially different in spirit from conventional asynchronous SGD settings \cite{lian2015asynchronous,agarwal2011distributed} where every learner updates the entire $\bm{\wts}$. In an alternate direction of work \cite{mania2017perturbed}, asynchrony is modelled as a perturbation. 
\subsection{OUR CONTRIBUTIONS}
Existing machine learning algorithms mostly try to optimize the trade-off of error with the number of iterations, epochs or ``work complexity'' \cite{bottou2016optimization}. Time to complete a task has traditionally been calculated in terms of work complexity measures \cite{sedgewick2011algorithms}, where the time taken to complete a task is a deterministic function of the size of the task (number of operations). However, due to straggling and synchronization bottle-necks in the system, the same task can often take different time to compute across different learners or iterations. We bring statistical perspective to the traditional work complexity analysis that incorporates the randomness introduced due to straggling. In this paper, we provide a systematic approach to analyze the expected error with runtime for both synchronous and asynchronous SGD, and some variants like $K$-sync, $K$-batch-sync, $K$-async and $K$-batch-async SGD by modelling the runtimes at each learner as i.i.d.\ random variables with a general distribution. 

We also propose a new error convergence analysis for Async and $K$-async SGD that holds for strongly convex objectives and can also be extended to non-convex formulations. In this analysis we relax the bounded delay assumption in \cite{lian2015asynchronous} and the bounded gradient assumption in \cite{recht2011hogwild}. We also remove the assumption of exponential computation time and the staleness process being independent of the parameter values \cite{mitliagkas2016asynchrony} as we will elaborate in \Cref{subsec:main_async_fixed}. Interestingly, our analysis also brings out the regimes where asynchrony can be better or worse than synchrony in terms of speed of convergence. Further, we propose a new learning rate schedule to compensate for staleness, and stabilize asynchronous SGD \textcolor{black}{that is related but different from momentum tuning in \cite{mitliagkas2016asynchrony, zhang2017yellowfin} as we clarify in \Cref{rem:variable}.}

The rest of the paper is organized as follows.  \Cref{sec:system model} describes our problem formulation introducing the system model and assumptions. \Cref{sec:main results} provides the main results of the paper -- analytical characterization of expected runtime and new convergence analysis for Async and $K$-async SGD and the proposed learning rate schedule to compensate for staleness. 
The analysis of expected runtime is elaborated further in \Cref{sec:runtime}. Proofs and detailed discussions are presented in the Appendix.  
%
\section{PROBLEM FORMULATION}
\label{sec:system model}
Our objective is to minimize the risk function of the parameter vector $\wts$ as mentioned in \cref{eq:minimization} given $N$ training samples. Let $S$ denote the total set of $N$ training samples, \textit{i.e.}, a collection of some data points with their corresponding labels or values. We use the notation $\xi$ to denote a random seed $ \in S$ which consists of either a single data and its label or a single mini-batch ($\mb$ samples) of data and their labels. 
\subsection{SYSTEM MODEL}
We assume that there is a central parameter server (PS) with $P$ parallel learners as shown in \Cref{fig:param_server}. The learners fetch the current parameter vector $\wts_j$ from the PS as and when instructed in the algorithm. Then they compute gradients using one mini-batch and push their gradients back to the PS as and when instructed in the algorithm. At each iteration, the PS aggregates the gradients computed by the learners and updates the parameter $\wts$. Based on how these gradients are fetched and aggregated, we have different variants of synchronous or asynchronous SGD.

\begin{figure}[t]
\centerline{\includegraphics[width=5cm]{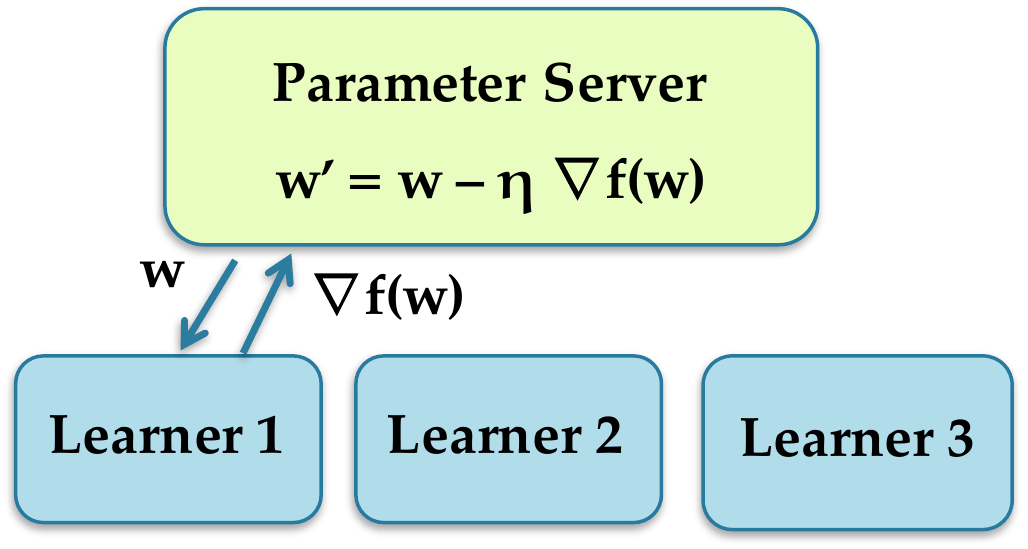}}
\caption{Parameter Server Model}
\label{fig:param_server}
\end{figure}

The time taken by a learner to compute gradient of one mini-batch is denoted by random variable $X_i$ for $i=1,2,\dots,P$. We assume that the $X_i$s are i.i.d.\ across mini-batches and learners.    

\subsection{PERFORMANCE METRICS}
There are two metrics of interest: \textcolor{black}{Expected Runtime and Error.}
\begin{defn}[Expected Runtime per iteration]
The expected runtime per iteration is the expected time (average time) taken to perform each iteration, \textit{i.e.}, the expected time between two consecutive updates of the parameter $\wts$ at the central PS.
\end{defn}

\begin{defn}[Expected Error]
The expected error after $j$ iterations is defined as $\E{F(\wts_j) - F^*}$, the expected gap of the risk function from its optimal value.
\end{defn}

Our aim is to determine the trade-off between the expected error (measures the accuracy of the algorithm) and the expected runtime \textbf{after a total of $J$ iterations} for the different SGD variants. 
\subsection{VARIANTS OF SGD}
We now describe the SGD variants considered in this paper. Please refer to \Cref{fig:ksync} and \Cref{fig:kasync} for a pictorial illustration.

\textbf{$K$-sync SGD}: This is a generalized form of synchronous SGD, also suggested in \cite{gupta2016model,chen2016revisiting} to offer some resilience to straggling as the PS does not wait for all the learners to finish. The PS only waits for the first $K$ out of $P$ learners to push their gradients. Once it receives $K$ gradients, it updates $ \wts_{j}$ and cancels the remaining learners. The updated parameter vector $\wts_{j+1}$ is sent to all $P$ learners for the next iteration. The update rule is given by:
\begin{equation}
\wts_{j+1} = \wts_j - \frac{\eta}{K}\sum_{l=1}^K g(\wts_{j}, \xi_{l,j}).
\label{eq:ksync}
\end{equation}
Here $l=1,2,\ldots,K$ denotes the index of the $K$ learners that finish first, $\xi_{l,j}$ denotes the mini-batch of $m$ samples used by the $l$-th learner at the $j$-th iteration and  $g(\wts_{j}, \xi_{l,j})= \frac{1}{m} \sum_{\xi \in \xi_{l,j} }  \nabla f(\wts_{j}, \xi)$ denotes the average gradient of the loss function evaluated over the mini-batch  $\xi_{l,j}$ of size $m$. For $K=P$, the algorithm is exactly equivalent to a fully synchronous SGD with $P$ learners. 

\textbf{$K$-batch-sync:} In $K$-batch-sync, all the $P$ learners start computing gradients with the same $\wts_j$. Whenever any learner finishes, it pushes its update to the PS and evaluates the gradient on the next mini-batch at the same $\wts_j$. The PS updates using the first $K$ mini-batches that finish and cancels the remaining learners. Theoretically, the update rule is still the same as \cref{eq:ksync} but here $l$ now denotes the index of the mini-batch (out of the $K$ mini-batches that finished first) instead of the learner. However $K$-batch-sync will offer advantages over $K$-sync in runtime per iteration as no learner is idle. 
 \begin{figure}[t]
\centerline{\includegraphics[height=3cm]{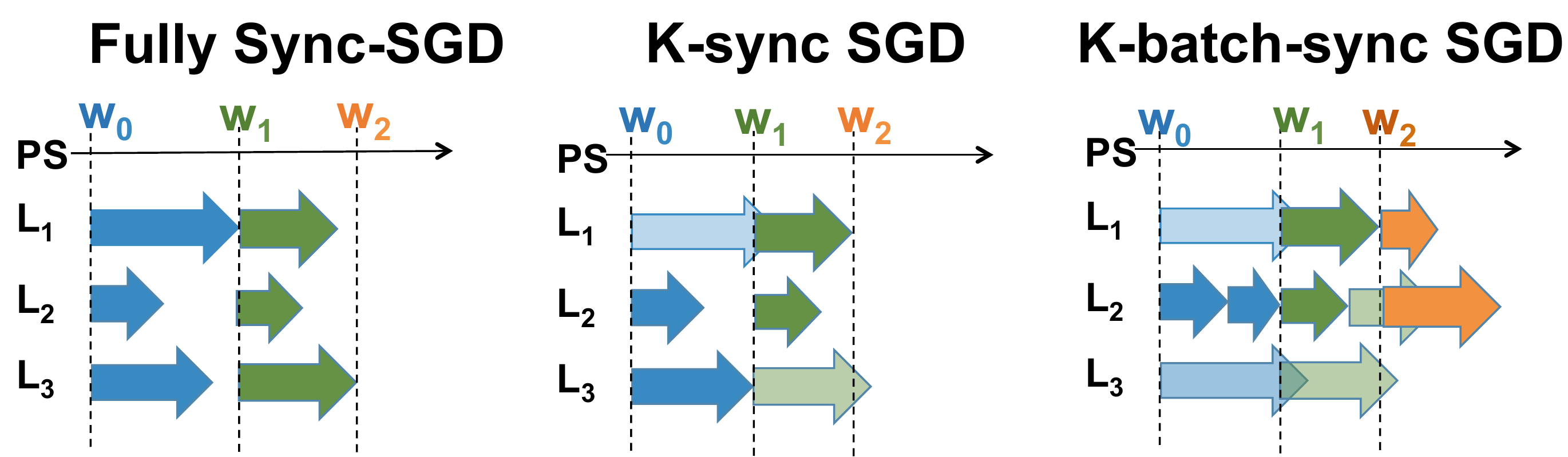}}
\caption{For $K=2$ and $P=3$, we illustrate the $K$-sync and $K$-batch-sync SGD in comparison with fully synchronous SGD. Lightly shaded arrows indicate straggling gradient computations that are cancelled.}
\label{fig:ksync}
\end{figure}
\begin{figure}[t]
\centerline{\includegraphics[height=3cm]{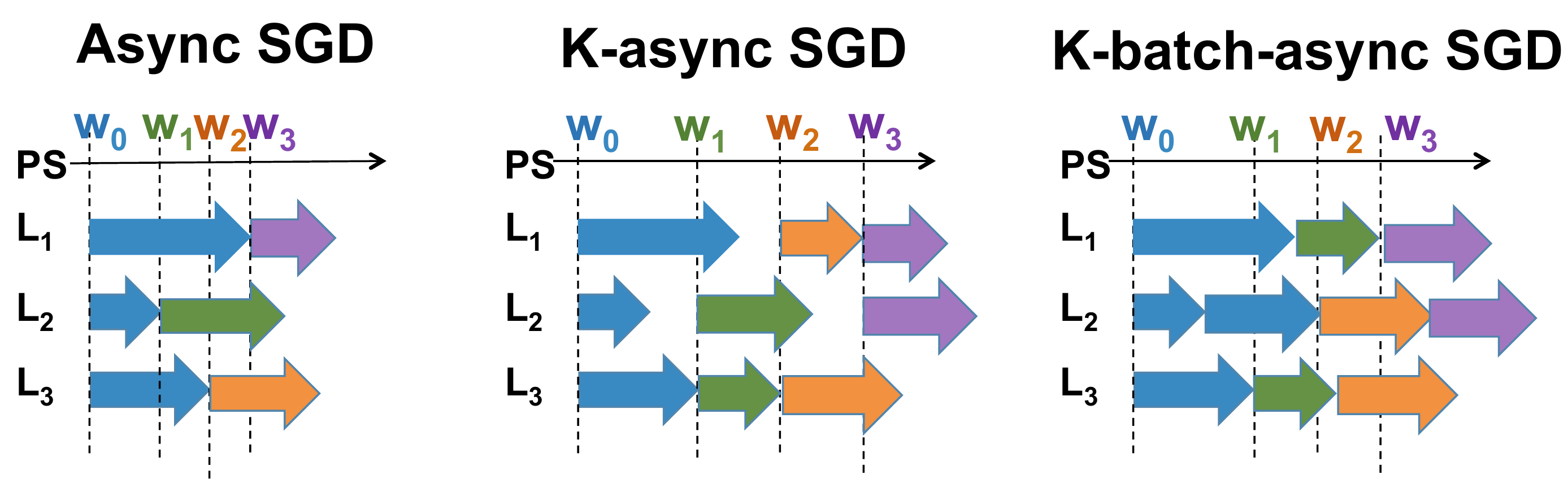}}
\caption{For $K=2$ and $P=3$, we illustrate the $K$-async and $K$-batch-async algorithms in comparison with fully asynchronous SGD.}
\label{fig:kasync}
\end{figure}

\textbf{$K$-async SGD}:
This is a generalized version of asynchronous SGD, also suggested in \cite{gupta2016model}. In $K$-async SGD, all the $P$ learners compute their respective gradients on a single mini-batch. The PS waits for the first $K$ out of $P$ that finish first, but it does not cancel the remaining learners. As a result, for every update the gradients returned by each learner might be computed at a stale or older value of the parameter $\wts$. The update rule is thus given by:
\begin{equation}
\wts_{j+1} = \wts_j - \frac{\eta}{K} \sum_{l=1}^K g(\wts_{\tau(l,j)}, \xi_{l,j}).
\label{eq:kasync}
\end{equation}
Here $l=1,2,\ldots,K$ denotes the index of the $K$ learners that contribute to the update at the corresponding iteration,  $\xi_{l,j}$ is one mini-batch of $m$ samples used by the $l$-th learner at the $j$-th iteration and $\tau(l,j)$ denotes the iteration index when the $l$-th learner last read from the central PS where $\tau(l,j) \leq j $.  Also, $g(\wts_{\tau(l,j)}, \xi_{l,j})= \frac{1}{m} \sum_{\xi \in \xi_{l,j} }  \nabla f(\wts_{\tau(l,j)}, \xi_{l,j})$ is the average gradient of the loss function evaluated over the mini-batch $\xi_{l,j}$ based on the stale value of the parameter $\wts_{\tau(l,j)}$.  For $K=1$, the algorithm is exactly equivalent to fully asynchronous SGD, and the update rule can be simplified as:
\begin{equation}
\wts_{j+1} = \wts_j - \eta g(\wts_{\tau(j)}, \xi_{j}).
\end{equation}
Here $\xi_{j}$ denotes the set of samples used by the learner that updates at the $j$-th iteration such that $|\xi_{j}| = m$ and $\tau(l,j)$ denotes the iteration index when that particular learner last read from the central PS. Note that $\tau(j) \leq j $.  

\textbf{$K$-batch-async:} Observe in \Cref{fig:kasync} that $K$-async also suffers from some learners being idle while others are still working on their gradients until any $K$ finish. In $K$-batch-async (proposed in \cite{lian2015asynchronous}), the PS waits for $K$ mini-batches before updating itself but irrespective of which learner they come from.  So wherever any learner finishes, it pushes its gradient to the PS, fetches current parameter at PS and starts computing gradient on the next mini-batch based on the current value of the PS. Surprisingly, the update rule is again similar to \cref{eq:kasync} theoretically except that now $l$ denotes the indices of the $K$ mini-batches that finish first instead of the learners and $\wts_{\tau(l,j)}$ denotes the version of the parameter when the learner computing the $l-$th mini-batch last read from the PS.  While the error convergence of $K$-batch-async is similar to $K$-async, it reduces the runtime per iteration as no learner is idle.

\begin{rem} 
\label{rem:redundancy_techniques}
Recent works such as \cite{tandon2017gradient} propose erasure coding techniques to overcome straggling learners. Instead, the SGD variants considered in this paper such as $K$-sync and $K$-batch-sync SGD exploit the inherent redundancy in the data itself, and ignore the gradients returned by straggling learners. If the data is well-shuffled such that it can be assumed to be i.i.d.\ across learners, then for the same effective batch-size, ignoring straggling gradients will give equivalent error scaling as coded strategies, and at a lower computing cost. However, coding strategies may be useful in the non i.i.d.\ case, when the gradients supplied by each learner provide diverse information that is important to capture in the trained model.
\end{rem}

\subsection{ASSUMPTIONS}
Closely following \cite{bottou2016optimization}, we also make the following assumptions:
\begin{enumerate}[leftmargin=*]
\item $F(\wts)$ is an $\lips-$ smooth function. Thus,
\begin{equation}
||\nabla F(\wts_1)- \nabla F(\wts_2)||_2 \leq \lips ||\wts_1 -\wts_2 ||_2. 
\end{equation}
\item $F(\wts)$ is strongly convex with parameter $c$. Thus,
\begin{equation}
\label{eq:strong-convexity}
2c(F(\wts)-F^* ) \leq ||\nabla F(\wts)||_2^2 \ \ \forall \  \wts.
\end{equation}
Refer to \Cref{sec:strong_convexity} for discussion on strong convexity. Our results also extend to non-convex objectives, as discussed in \Cref{sec:main results}.
\item The stochastic gradient is an unbiased estimate of the true gradient:
\begin{equation}
\Esub{\xi_{j}|
\wts_k}{g(\wts_k,\xi_{j})}= \nabla F(\wts_k) \ \ \forall \  k \leq j.
\end{equation}
Observe that this is slightly different from the common assumption that says $\Esub{\xi_{j}}
{g(\wts,\xi_{j})}= \nabla F(\wts)$ for all $\wts$.
Observe that all $\wts_j$ for $j>k$ is actually not independent of the data $\xi_{j}$.  We thus make the assumption more rigorous by conditioning on $\wts_k$ for $k \leq j$. Our requirement $k \leq j$ means that $\wts_k$ is the value of the parameter at the PS before the data $\xi_{j}$ was accessed and can thus be assumed to be independent of the data $\xi_{j}$.

\item Similar to the previous assumption, we also assume that the variance of the stochastic update given $\wts_k$ at iteration $k$ before the data point was accessed is also bounded as follows:
\begin{align}
&\Esub{\xi_{j}|\wts_k}{||g(\wts_k,\xi_{j})- \nabla F(\wts_k) ||_2^2} 
\leq \frac{\sigma^2}{m} + \frac{M_G}{m} ||\nabla F(\wts_k) ||_2^2 \ \forall \ k \leq j.
\end{align}
\end{enumerate}

In the following \Cref{sample-table}, we provide a list of the notations used in this paper for referencing.

\begin{table}[!ht]
\begin{center}
\begin{tabular}{llll}
CONSTANTS &  & RANDOM VARIABLES\\
\hline 
Mini-batch Size         & $m$ & Runtime of a learner for one mini-batch & $X_i$\\
Total Iterations            & $J$  & Runtime per iteration & $T$\\
Number of learners (Processors) & $P$ \\
Number of learners to wait for & $K$ \\
Learning rate & $\eta$ \\
Lipschitz Constant & $L$ \\
Strong-convexity parameter & $c$\\
\hline
\end{tabular}
\end{center}
\caption{LIST OF NOTATIONS}
 \label{sample-table}
\end{table}

\section{MAIN RESULTS}
\label{sec:main results}
\subsection{RUNTIME ANALYSIS}
\label{subsec:main_runtime}
We compare the theoretical wall clock runtime of the different SGD variants to illustrate the speed-up offered by different asynchronous and batch variants. A detailed discussion is provided in \Cref{sec:runtime}.  

\begin{thm} Let the wall clock time of each learner to process a single mini-batch be i.i.d.\ random variables $X_1, X_2,\dots,X_P$. Then the ratio of the expected runtimes per iteration for synchronous and asynchronous SGD is
$$ \frac{\E{T_{Sync}}}{\E{T_{Async}}}=P \frac{\E{X_{P:P}}}{\E{X}}$$
where $X_{(P:\learners)}$ is the $P^{th}$ order statistic of $P$ i.i.d.\ random variables $X_1, X_2, \dots , X_{\learners}$. 
\label{thm:runtime1}
\end{thm}
This result analytically characterizes the speed-up offered by asynchronous SGD for \textit{any general distribution on the wall clock time of each learner}. To prove this result, we use ideas from renewal theory, as we discuss in \Cref{sec:runtime}. In the following corollary, we highlight this speed-up for the special case of exponential computation time.
\begin{coro}
\label{coro:syncAsync}
Let the wall clock time of each learner to process a single mini-batch be i.i.d.\ exponential random variables $X_1, X_2,\dots,X_P \sim \exp({\mu})$. Then the ratio of the expected runtimes per iteration for synchronous and asynchronous SGD is approximately given by $P \log{P}$.
\end{coro}
Thus, the speed-up scales with $P$ and can diverge to infinity for large $P$. We illustrate the speed-up for different distributions in \Cref{fig:syncAsync}. It might be noted that a similar speed-up as \Cref{coro:syncAsync} has also been obtained in a recent work \cite{hannah2017more} under exponential assumptions.

 \begin{figure}[t]
\centerline{\includegraphics[height=4.5cm]{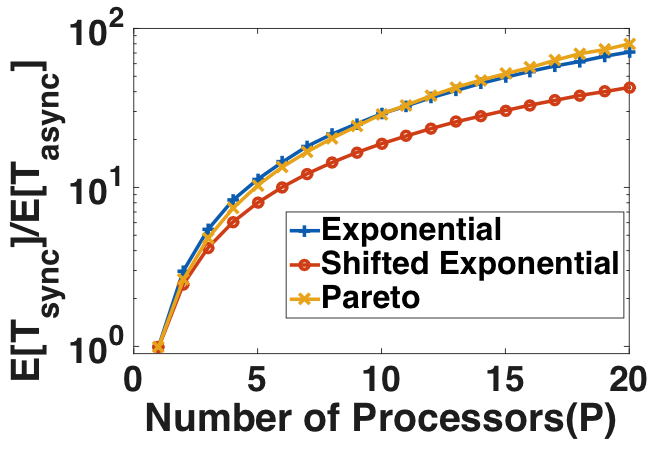}}
\caption{Plot of the speed-up using asynchronous over synchronous: $\log{\frac{\E{T_{Sync}}}{\E{T_{Async}}}}$ with $P$ for different distributions - $ \exp(1)$, $1+ \exp(1)$ and $Pareto(2,1)$.}
\label{fig:syncAsync}
\end{figure}

The next result illustrates the advantages offered by $K$-batch-sync and async over their corresponding counterparts $K$-sync and $K$-async respectively.
\begin{thm}
\label{thm:runtime2}
Let the wall clock time of each learner to process a single mini-batch be i.i.d.\  exponential random variables $X_1, X_2,\dots,X_P \sim \exp({\mu})$. Then the ratio of the expected runtimes per iteration for $K$-async (or sync) SGD and $K$-batch-async (or sync) SGD is 
$$ \frac{\E{T_{K-async}}}{\E{T_{K-batch-async}}} =\frac{P\E{X_{K:P} } }{K\E{X} }  \approx \frac{P \log{\frac{P}{P-K} } }{K}  $$
where $X_{K:P}$ is the $K^{th}$ order statistic of i.i.d.\ random variables $X_1,X_2,\dots,X_P$.
\end{thm}

To prove this, we derive an exact expression (see \Cref{lem:runtime kbatch} in \Cref{sec:runtime}) for the expected runtime of $K$-batch-async SGD, for \textit{any} given i.i.d.\ distribution of $X_i$s, not necessarily exponential. The expected runtime per iteration is obtained as $\frac{K\E{X}}{P}$, using ideas from renewal theory. The full proof of \Cref{thm:runtime2} is also provided in \Cref{sec:runtime}.

\Cref{thm:runtime2} shows that as $\frac{K}{P}$ increases, the speed-up using $K$-batch-async increases and can be upto $\log{P}$ times higher. For non-exponential distributions, we simulate the behaviour of expected runtime in \Cref{fig:runtime2} for $K$-sync, $K$-async and $K$-batch-async respectively for Pareto and Shifted Exponential.

\begin{figure}[t]
\centering
\includegraphics[width=13cm]{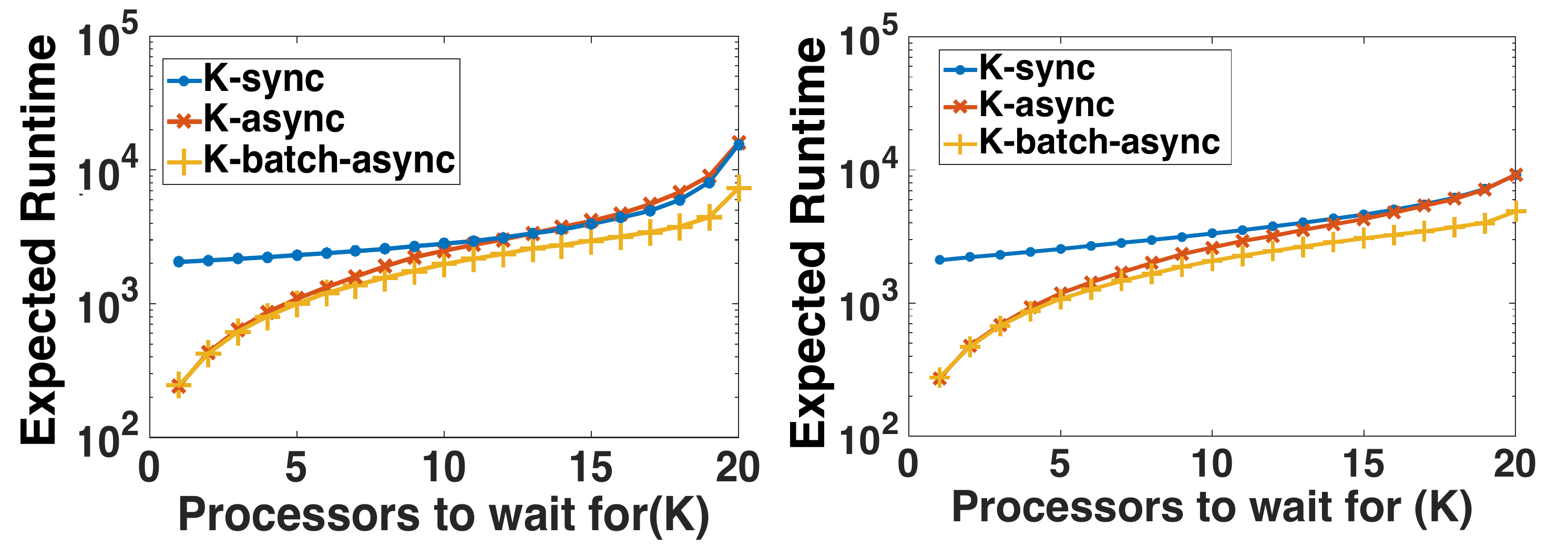}
\caption{Plot of expected runtime for $2000$ iterations: (Left) Pareto distribution $Pareto(2,1)$ and (Right) Shifted exponential distribution $1+ \exp(1)$.}
\label{fig:runtime2}
\end{figure}

\subsection{ERROR ANALYSIS UNDER FIXED LEARNING RATE}
\label{subsec:main_async_fixed}
\Cref{thm:error kasync} below gives a convergence analysis of $K$-async SGD for fixed $\eta$, relaxing the following assumptions in existing literature.
\begin{itemize}[leftmargin=*]
\item In several prior works such as \cite{mitliagkas2016asynchrony,lee2017speeding,dutta2016short,hannah2017more}, it is often assumed, for the ease of analysis, that runtimes are exponentially distributed. In this paper, we extend our analysis for any general service time $X_i$.
\item In \cite{mitliagkas2016asynchrony}, it is also assumed that the staleness process is independent of $\wts$. While this assumption simplifies the analysis greatly, it is not true in practice. For instance, for a two learner case, the parameter $\wts_2$ after $2$ iterations depends on whether the update from $\wts_1$ to $\wts_2$ was based on a stale gradient at $\wts_0$ or the current gradient at $\wts_1$, depending on which learner finished first. In this work, we remove this independence assumption.
\item Instead of the bounded delay assumption in \cite{lian2015asynchronous}, we use a general staleness bound $$\E{|| \nabla F(\wts_{j}) - \nabla F(\wts_{\tau(l,j)})||_2^2 } \leq \gamma \E{|| \nabla F(\wts_{j}) ||_2^2 }$$ which allows for large, but rare delays.
\item In \cite{recht2011hogwild}, the norm of the gradient is assumed to be bounded. However, if we assume that $||\nabla F(\wts) ||_2^2 \leq M$ for some constant $M$, then using \cref{eq:strong-convexity} we obtain $ ||\wts-\wts^*||_2^2 \leq \frac{2}{c} (F(\wts)-F^*) \leq \frac{M}{c^2} $ implying that $\wts$ itself is bounded which is a very strong and restrictive assumption, that we relax in this result.
\end{itemize}

\noindent Some of these assumptions have been addressed in the context of alternative asynchronous SGD variants in the recent works of \cite{hannah2017more,hannah2016unbounded,sun2017asynchronous,leblond2017asaga}.

\begin{thm}
\label{thm:error kasync}
Suppose the objective $F(\wts)$ is $c$-strongly convex and the learning rate $\eta \leq \frac{1}{2\lips\left(\frac{M_G}{Km}+ \frac{1}{K}\right)}$. Also assume that for some $\gamma \leq 1$, $$\E{|| \nabla F(\wts_{j}) - \nabla F(\wts_{\tau(l,j)})||_2^2 } \leq \gamma \E{|| \nabla F(\wts_{j}) ||_2^2 }.$$ Then, the error of $K$-async SGD after $J$ iterations is,
\begin{align}
& \E{F(\wts_{J})}-F^* 
\leq \frac{\eta L\sigma^2}{2c\gamma'K m} 
+ (1 - \eta c\gamma')^{J} \left(\E{F(\wts_{0})}-F^* - \frac{\eta L\sigma^2}{2c \gamma' Km}  \right)
\end{align}
where $\gamma'= 1-\gamma + \frac{p_0}{2}$ and $p_0$ is a lower bound on the conditional probability that $\tau(l,j)=j$, given all the past delays and parameters.
\end{thm}

Here, $\gamma$ is a measure of staleness of the gradients returned by learners; smaller $\gamma$ indicates less staleness. 

The full proof is provided in \Cref{sec:async_proof}. We first prove the result for $K=1$ in \Cref{subsec:async_proof} for ease of understanding, and then provide the more general proof for any $K$ in \Cref{subsec:K_async_proof}.
We use \Cref{lem:delay} below to prove \Cref{thm:error kasync}. 
\begin{lem} 
\label{lem:delay}
Suppose that $p_0^{(l,j)}$ is the conditional probability that $\tau(l,j)=j$ given all the past delays and all the previous $\wts$, and $p_0 \leq p_0^{(j)}$ for all $j$. Then,
\begin{equation}
\E{||\nabla F(\wts_{\tau(l,j)})||_2^2} \geq p_0 \E{||\nabla F(\wts_{j})||_2^2}.
\end{equation}
\end{lem}
\begin{proof}
By the law of total expectation,
\begin{align*}
\E{||\nabla F(\wts_{\tau(l,j)})||_2^2} &= p_0^{(l,j)} \E{||\nabla F(\wts_{\tau(l,j)})||_2^2|\tau(j)= j}    
+ (1-p_0^{(l,j)}) \E{||\nabla F(\wts_{\tau(l,j)})||_2^2|\tau(j) \neq j} \nonumber \\
&  \geq p_0 \E{||\nabla F(\wts_{j})||_2^2}. \hspace{2.2cm} 
\end{align*}
\end{proof}

For the exponential distribution, $p_0$ is equal to $\frac{1}{P}$ as we discuss in \Cref{lem:p_0}. For non-exponential distributions, it is a constant in $[0,1]$. For some special classes of distributions like new-longer-than-used (new-shorter-than-used) as defined in \Cref{defn:new-longer-than-used}, we can formally show that $p_0$ lies in $[0,\frac{1}{P}] $ ($[\frac{1}{P},1]$) respectively. The following \Cref{lem:p_0} below provides bounds on $p_0$.

\begin{lem}[Bounds on $p_0$]
\label{lem:p_0}
Define $p_0=\inf_{j}p_0^{(j)}$, \textit{i.e.} the largest constant such that $p_0 \leq p_0^{(j)} \ \forall \ j$.
\begin{itemize}
\item For exponential computation times, $p_0^{(j)} = \frac{1}{P}$ for all $j$ and is thus invariant of $j$ and $p_0=\frac{1}{P}$.
\item For new-longer-than-used (See \Cref{defn:new-longer-than-used}) computation times, $p_0^{(j)} \leq \frac{1}{P} $ and thus $p_0 \leq \frac{1}{P} $.
\item For new-shorter-than-used computation times, $p_0^{(j)} \geq \frac{1}{P} $ and thus $p_0 \geq \frac{1}{P} $.
\end{itemize}
\end{lem}
The proof is provided in \Cref{subsec:proof_lem_p_0}.

For $K$-batch-async, the update rule is same as $K$-async except that the index $l$ denotes the index of the mini-batch. Thus, the error analysis will be exactly similar. Our analysis can also be extended to non-convex $F(\wts)$ as we show in \Cref{subsec:non_convex}.

Now let us compare with $K$-sync SGD. We observe that the analysis of $K$-sync SGD is same as serial SGD with mini-batch size $Km$. Thus,
\begin{lem}[Error of $K$-sync]
 \cite{bottou2016optimization}
Suppose that the objective $F(\wts)$ is $c$-strongly convex and learning rate $\eta \leq \frac{1}{2\lips(\frac{M_G}{Km}+ 1)}$. Then, the error after $J$ iterations of $K$-sync SGD is
\begin{align*}
\E{F(\wts_{\iters}) - F^*}&\leq \frac{ \eta \lips \sigma^2}{2 c (K\mb)} +  (1- \eta c )^{J} \left( F(\wts_0) - F^* - \frac{ \eta \lips \sigma^2}{2 c (K\mb)} \right).
\end{align*}
\label{lem:error ksync}
\end{lem}

\begin{figure}[t]
\centerline{\includegraphics[height=4cm]{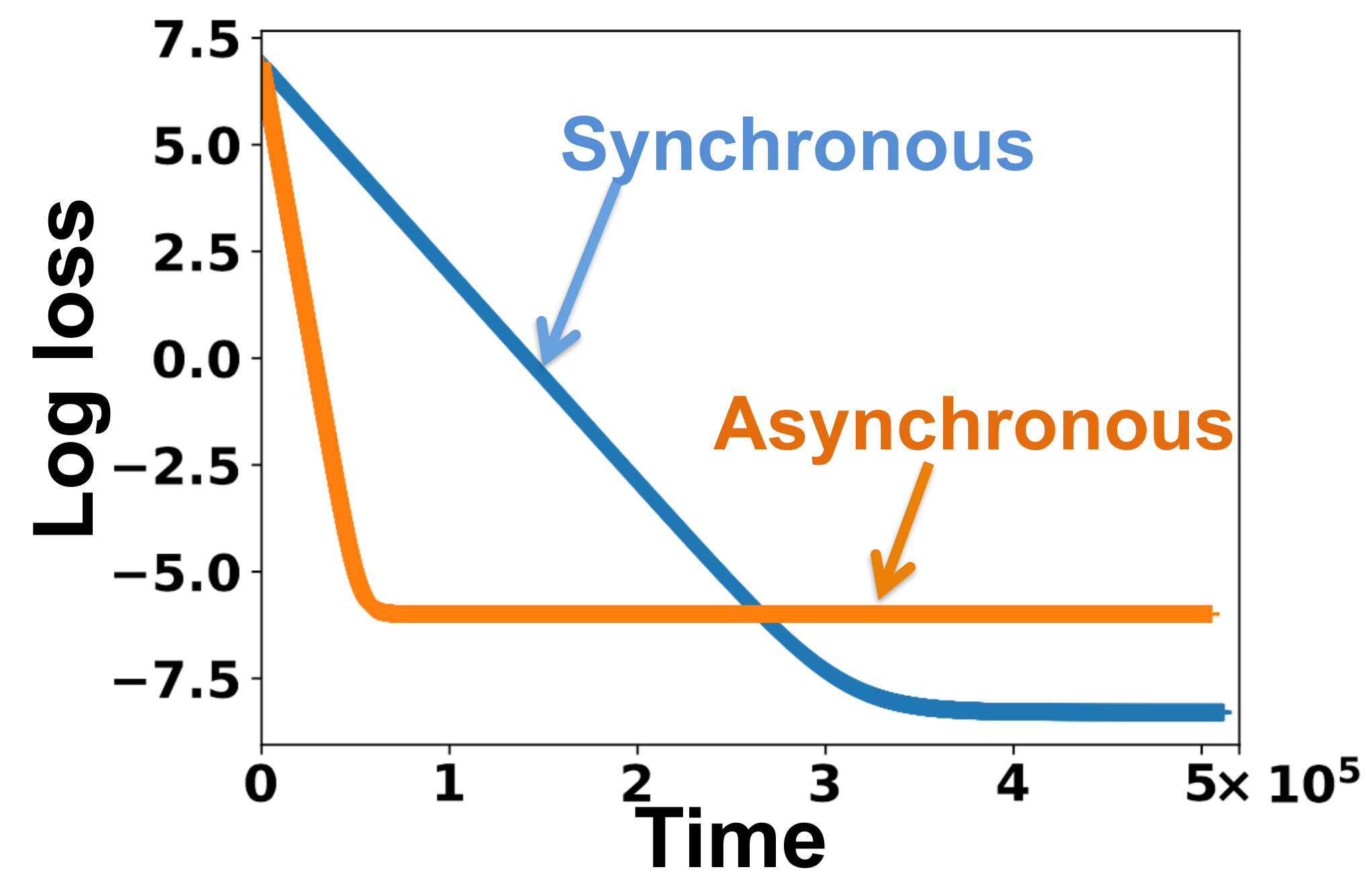}}
\caption{Theoretical error-runtime trade-off for Sync and Async-SGD with same $\eta$. Async-SGD has faster decay with time but a higher error floor.}
\label{fig:theoreticalsyncAsync}
\end{figure}

\textit{Can stale gradients win the race?}
For the same $\eta$, observe that the error given by \Cref{thm:error kasync} decays at the rate $(1 - \eta c(1-\gamma + \frac{p_0}{2}))$ for $K$-async or $K$-batch-async SGD while for $K$-sync, the decay rate with number of iterations is $(1 - \eta c)$. Thus, depending on the values of $\gamma$ and $p_0$, the decay rate of $K$-async or $K$-batch-async SGD can be faster or slower than $K$-sync SGD. The decay rate of $K$-async or $K$-batch-async SGD is faster if $\frac{p_0}{2}>\gamma $. As an example, one might consider an exponential or new-shorter-than-used service time where $p_0 \geq \frac{1}{P}$ and $\gamma$ can be made smaller by increasing $K$. It might be noted that asynchronous SGD can still be faster than synchronous SGD with respect to wall clock time even if its decay rate with respect to number of iterations is lower as every iteration is much faster in asynchronous SGD (Roughly $P\log{P}$ times faster for exponential service times).

The maximum allowable learning rate for synchronous SGD is $\max\{ \frac{1}{c}, \frac{1}{2L (\frac{M_G}{Pm}+1)}  \}$ which can be much higher than that for asynchronous SGD,\textit{i.e.}, $\max\{ \frac{1}{c(1-\gamma + \frac{p_0}{2})}, \frac{1}{2L (\frac{M_G}{m}+1)}  \}$. Similarly the error-floor for synchronous is $ \frac{\eta L\sigma^2}{2c Pm} $ as compared to asynchronous whose error floor is $ \frac{\eta L\sigma^2}{2c(1-\gamma + \frac{p_0}{2})m}$.

In \Cref{fig:theoreticalsyncAsync}, we compare the theoretical trade-offs between synchronous ($K=P$ in \Cref{lem:error ksync}) and asynchronous SGD ($K=1$ in \Cref{thm:error kasync}). Async-SGD converges very quickly, but to a higher floor. \Cref{fig:MNIST_time} shows the same comparison on the MNIST dataset, along with $K$-batch-async SGD. 

\begin{figure}[t]
\centerline{\includegraphics[height=4cm]{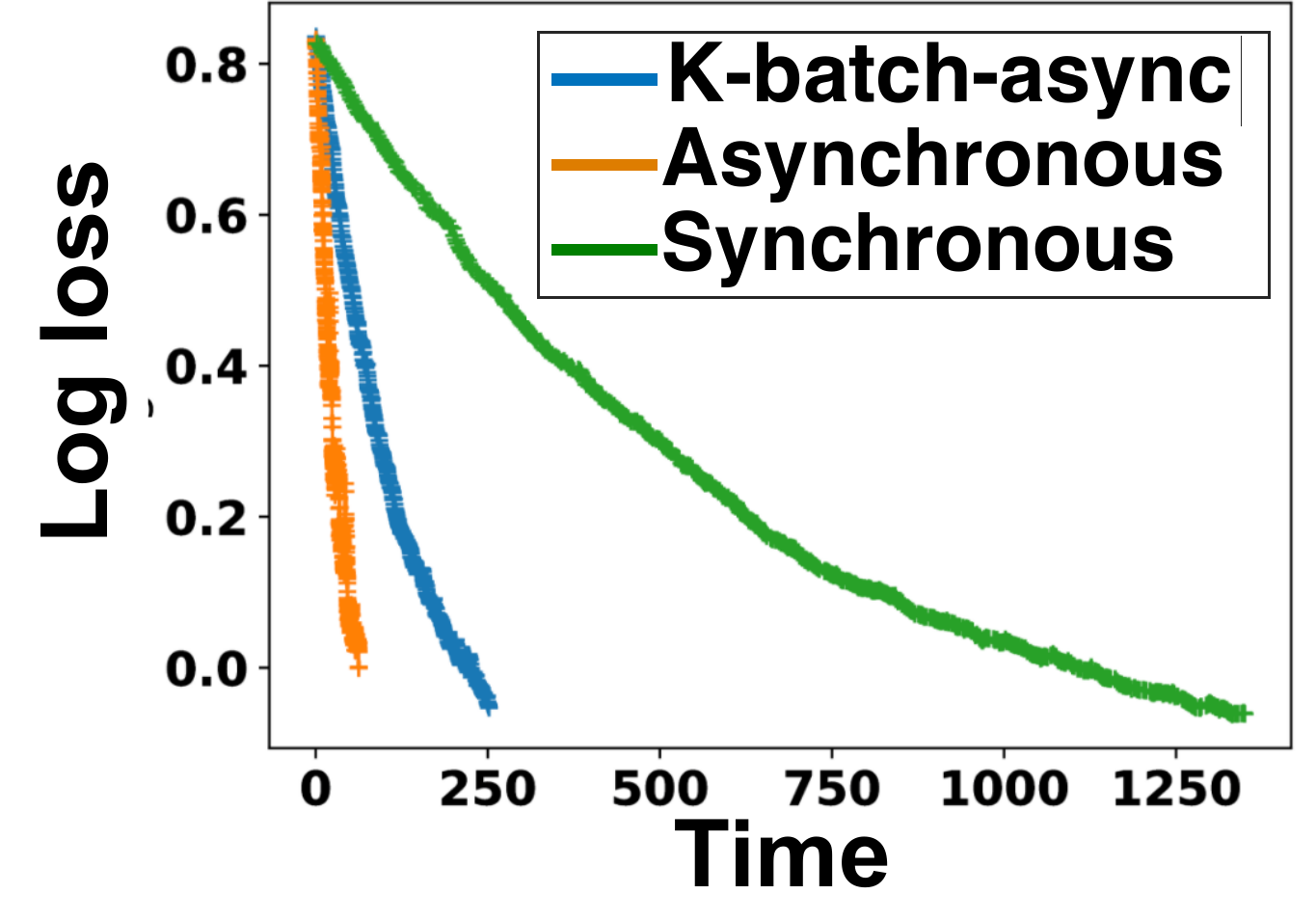}}
\caption{Error-runtime trade-off comparison of different SGD variants for logistic regression on MNIST, with $X_i \sim \exp(1)$, $P=8$, $K=4$, $\eta=0.01$ and $m=1$. $K$-batch-async gives intermediate performance, between Async and sync-SGD. (Details of setup provided in \Cref{sec:simulation_setup}.)
}
\label{fig:MNIST_time}
\end{figure}
\subsection{VARIABLE LEARNING RATE FOR STALENESS COMPENSATION}
The staleness of the gradient is random, and can vary across iterations. Intuitively, if the gradient is less stale, we want to weigh it more while updating the parameter $\wts$, and if it is more stale we want to scale down its contribution to the update. With this motivation, we propose the following condition on the learning rate at different iterations.
\begin{equation}
\eta_j \E{ ||\wts_{j}- \wts_{\tau(j)} ||_2^2} \leq C 
\label{eq:learning rate condition}
\end{equation}
for a constant $C$. This condition is also inspired from our error analysis in \Cref{thm:error kasync}, because it helps remove the assumption $\E{|| \nabla F(\wts_{j}) - \nabla F(\wts_{\tau(j)})||_2^2 } \leq \gamma \E{|| \nabla F(\wts_{j}) ||_2^2 }$. Using \cref{eq:learning rate condition}, we obtain the following convergence result.

\begin{thm}
Suppose the learning rate in the $j$-th iteration $\eta_j \leq 1/2L(\frac{M_G}{m}+1) $, and $$\eta_j \E{||\wts_j - \wts_{\tau(j)} ||_2^2 } \leq C  $$ for some constant $C$. Then, we have
\begin{align}
\E{F(\wts_{J})} - F^* &\leq \Delta + (\E{F(\wts_{0})}- F^*) \prod_{j=1}^{J}(1-\rho_j) \nonumber
\end{align}
where $\rho_j=\eta_j(1+ \frac{p_0}{2})c$, and the error floor  $\Delta= \Delta_J + (1-\rho_J)\Delta_{J-1} +  \dots + \prod_{j=1}^{J}(1-\rho_j)\Delta_{0} $, where $\Delta_j=  \frac{\eta_j^2 L\sigma^2}{2m} +  \frac{CL^2}{2} $.  
\label{thm:variable learning rate}
\end{thm}
The proof is provided in \Cref{subsec:variable learning rate}. In our analysis of Asynchronous SGD, we observe that the term $\frac{\eta}{2}\E{|| \nabla F(\wts_{j}) - \nabla F(\wts_{\tau(j)})||_2^2 }$ is the most difficult to bound. For fixed learning rate, we had assumed that $\E{|| \nabla F(\wts_{j}) - \nabla F(\wts_{\tau(j)})||_2^2 }$ is bounded by $\gamma ||\nabla F(\wts_{j}) ||_2^2$. However, if we impose the condition \cref{eq:learning rate condition} on $\eta$, we do not require this assumption. Our proposed condition actually provides a bound for the staleness term as follows:
\begin{align}
 \frac{\eta_j}{2}&\E{|| \nabla F(\wts_{j}) - \nabla F(\wts_{\tau(j)})||_2^2 }  \leq \frac{\eta_j L^2}{2}\E{|| \wts_{j} - \wts_{\tau(j)}||_2^2 } \leq \frac{CL^2}{2}.
\label{eq:stalenessbound}
\end{align}

\begin{figure}[t]
\centerline{\includegraphics[height=4cm]{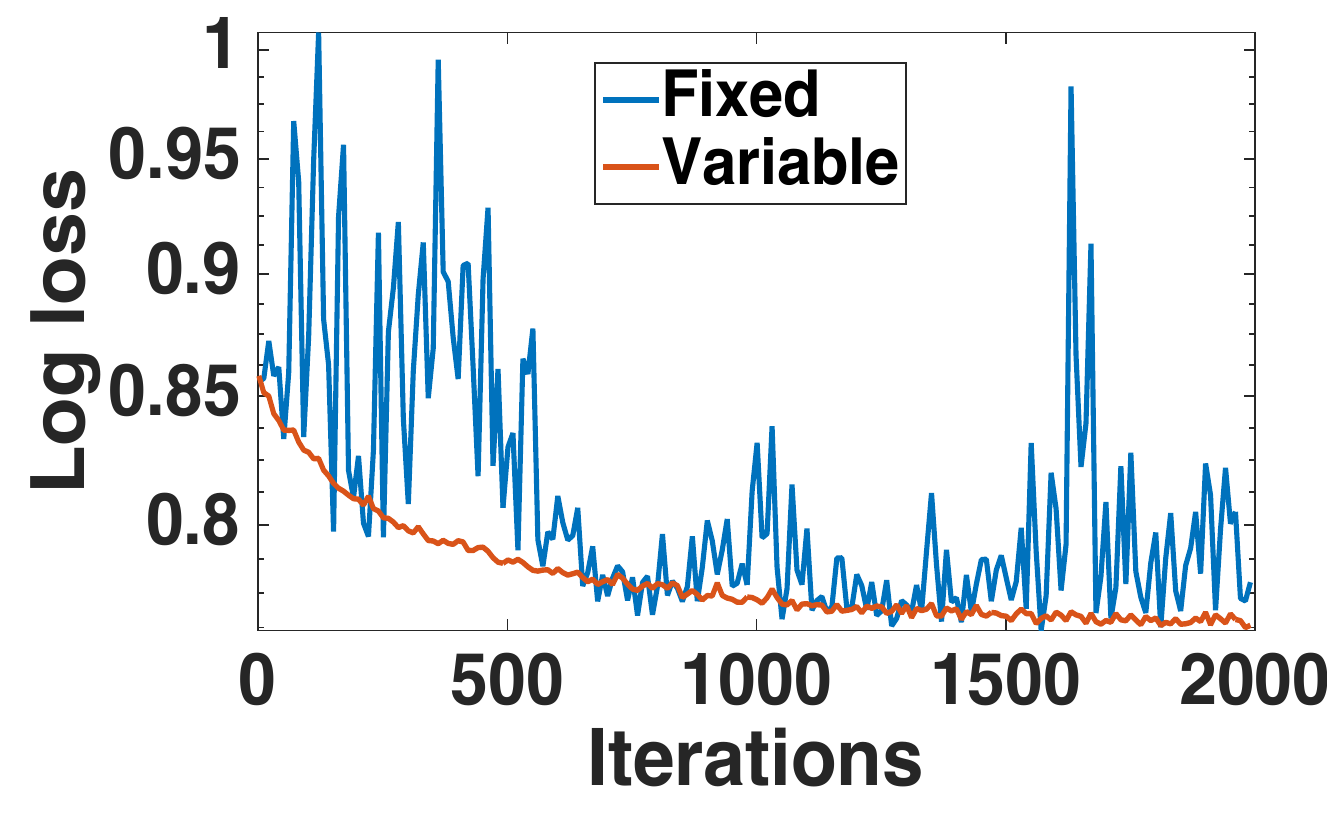}}
\caption{Async-SGD on CIFAR10 dataset, with $X \sim \exp{20}$, mini-batch size $m=250$ and $P=40$ learners. We compare fixed $\eta=0.01$, and the variable schedule given in \eqref{eqn:var_eta} for $\eta_{max}=0.01$ and $C=0.005\eta_{max}$. Observe that the proposed schedule can give fast convergence, and also maintain stability, while the fixed $\eta$ algorithm becomes unstable.
}
\label{fig:var_eta}
\end{figure}

\noindent \textbf{Proposed Algorithmic Modification} Inspired by this analysis, we propose the learning rate schedule,
\begin{equation}
\eta_j = \min\left\{ \frac{C}{||\wts_j-\wts_{\tau(j)} ||_2^2}, \eta_{max}
\right\} \label{eqn:var_eta}
\end{equation}
where $\eta_{max}$ is a suitably large ceiling on learning rate. It ensures stability when the first term in \eqref{eqn:var_eta} becomes large due to the staleness $||\wts_j-\wts_{\tau(j)} ||_2$ being small. The $C$ is chosen of the same order as the desired error floor. To implement this schedule, the PS needs to store the last read model parameters for every learner. In \Cref{fig:var_eta} we illustrate how this schedule can stabilize asynchronous SGD. We also show simulation results that characterize the performance of this algorithm in comparison with naive asynchronous SGD with fixed learning rate.

\begin{rem}
\label{rem:variable}
The idea of variable learning rate is related to the idea of momentum tuning in \cite{mitliagkas2016asynchrony, zhang2017yellowfin} and may have a similar effect of stabilizing the convergence of asynchronous SGD. However, learning rate tuning is arguably more general since asynchrony results in a momentum term in the gradient update (as shown in \cite{mitliagkas2016asynchrony, zhang2017yellowfin}) only under the assumption that the staleness process is geometric and independent of $\wts$.
\end{rem}

\section{RUNTIME ANALYSIS}
\label{sec:runtime}



In this section, we provide our analysis of the expected runtime of different variants of SGD. These lemmas are then used in the proofs of \Cref{thm:runtime1} and \Cref{thm:runtime2}. 
\subsection{RUNTIME OF $K$-SYNC SGD}
\begin{lem}[Runtime of $K$-sync SGD]
\label{lem:runtime ksync}
The expected runtime per iteration for $K$-sync SGD is,
\begin{align}
\E{T} &=  \E{X_{K:\learners}} 
\end{align}
where $X_{K:\learners}$ is the $K^{th}$ order statistic of $P$ i.i.d.\ random variables $X_1, X_2, \dots , X_{\learners}$.
\end{lem}

\begin{proof}[Proof of \Cref{lem:runtime ksync}]
We assume that the $P$ learners have an i.i.d.\ computation times. When all the learners start together, and we wait for the first $K$ out of $P$ i.i.d.\ random variables to finish, the expected computation time for that iteration is $\E{X_{K:P}}$, where $X_{K:P}$ denotes the $K$-th statistic of $P$ i.i.d.\ random variables $X_1,X_2,\dots,X_P$. 
\end{proof}
\textcolor{black}{Thus, for a total of $J$ iterations, the expected runtime is given by $J\E{X_{K:P}}$.}

\begin{rem} For  $X_i \sim \exp(\mu)$, the expected runtime per iteration is given by, $$\E{T} =\frac{1}{\mu} \sum_{i=P-K+1}^P \frac{1}{i} \approx \frac{1}{\mu} \left( \frac{ \log{\frac{P}{P-K}}}{\mu} \right)
$$ where the last step uses an approximation from \cite{sheldon2002first}. For justification, the reader is referred to \Cref{subsec:runtime_K_statistic}.
\end{rem}

\subsection{RUNTIME OF $K$-BATCH-SYNC SGD}
The expected runtime of $K$-batch-sync SGD is not analytically tractable in general, but for $X_i\sim \exp(\mu)$, the runtime per iteration is distributed as $Erlang(K, P\mu)$. Refer to \Cref{subsec:runtime_K_batch_sync} for explanation. Thus, for $K$-batch-sync SGD, the expected time per iteration is given by, $$\E{T}=  \frac{K}{P\mu}. $$

\subsection{RUNTIME OF $K$-BATCH-ASYNC SGD}
\begin{lem}[Runtime of $K$-batch-async SGD]
\label{lem:runtime kbatch}
The expected runtime per iteration for $K$-batch-async SGD in the limit of large number of iterations is given by:
\begin{equation}
\E{T} = \frac{K\E{X}}{P}.
\end{equation}
\end{lem}
Unlike the results for the synchronous variants, this result on average runtime per iteration holds only in the limit of large number of iterations. To prove the result we use ideas from renewal theory. For a brief background on renewal theory, the reader is referred to \Cref{subsec:runtime_K_batch_async}. 

\begin{proof}[Proof of \Cref{lem:runtime kbatch}]
For the $i$-th learner, let $\{N_i(t), t>0\}$ be the number of times the $i$-th learner pushes its gradient to the PS over in time $t$. The time between two pushes is an independent realization of $X_i$. Thus, the inter-arrival times $X_i^{(1)}, X_i^{(2)},\dots$ are i.i.d.\ with mean inter-arrival time $\E{X_i}$. Using the elementary renewal theorem \cite[Chapter 5]{gallager2013stochastic} we have,
\begin{equation}
\lim_{t \to \infty} \frac{\E{N_i(t)}}{t} = \frac{1}{\E{X_i}}.
\end{equation}
Thus, the rate of gradient pushes by the $i$-th learner is $1/\E{X_i}$. As there are $P$ learners, we have a superposition of $P$ renewal processes and thus the average rate of gradient pushes to the PS is 
\begin{equation}
\lim_{t \to \infty} \sum_{i=1}^P\frac{\E{N_i(t)}}{t} = \sum_{i=1}^P \frac{1}{\E{X_i}} = \frac{P}{\E{X}}.
\end{equation}

Every $K$ pushes are one iteration. Thus, the expected runtime per iteration or effectively the expected time for $K$ pushes is given by
$
 \E{T}= \frac{K\E{X}}{P}.
$ \end{proof}

\noindent Thus, for a total of $J$ iterations, the average runtime can be approximated as
$ \frac{JK\E{X}}{P}$ when $J$ is large. Note that Fully-Synchronous SGD is actually $K$-sync SGD with $K=P$, \textit{i.e.}, waiting for all the $P$ learners to finish. On the other hand, Fully-Asynchronous SGD is actually $K$-batch-async with $K=1$. Now, we provide the proofs of \Cref{thm:runtime1} and \Cref{coro:syncAsync} respectively, that provide a comparison between these two variants.

\begin{proof}[\textbf{Proof of \Cref{thm:runtime1}}] By taking the ratio of the expected runtimes per iteration in \Cref{lem:runtime ksync} with $K=P$ and \Cref{lem:runtime kbatch} with $K=1$, we get the result in \Cref{thm:runtime1}. 
\end{proof}


\begin{proof}[\textbf{Proof of \Cref{coro:syncAsync}}] The expectation of the maximum of $P$ i.i.d.\ $X_i \sim \exp(\mu)$ is $\E{X_{P:P}}=\sum_{i=1}^P \frac{1}{i\mu} \approx \frac{\log{P}}{\mu}$ \cite{sheldon2002first}. This can be substituted in \Cref{thm:runtime1} to get \Cref{coro:syncAsync}.
\end{proof}

\subsection{RUNTIME OF $K$-ASYNC SGD}
The expected runtime per iteration of $K$-async SGD is not analytically tractable for non-exponential $X_i$, but we obtain an upper bound on it for a class of distributions called the ``new-longer-than-used'' distributions, as defined below.

\begin{defn}[New-longer-than-used]
\label{defn:new-longer-than-used}
A random variable is said to have a new-longer-than-used distribution if the following holds for all $t, u \geq 0$:
$$
\Pr(U>u+t| U>t) \leq \Pr(U>u).
$$
\end{defn}
Most of the continuous distributions we encounter like normal, exponential, gamma, beta are new-longer-than-used. Alternately, the hyper exponential distribution is new-shorter-than-used and it satisfies $
\Pr(U>u+t | U>t) \geq \Pr(U>u)$ for all $t, u \geq 0$.

\begin{lem}[Runtime of $K$-async SGD]
\label{lem:runtime kasync}
Suppose that each $X_i$ has a new-longer-than-used distribution. Then, the expected \textcolor{black}{runtime per iteration} for $K$-async is upper-bounded as
\begin{align}
\E{T} & \leq   \E{X_{K:\learners}} \label{eqn:T_k_async}
\end{align}
where $X_{K:\learners}$ is the $K^{th}$ order statistic of $P$ i.i.d.\ random variables $X_1, X_2, \dots , X_{\learners}$. 
\end{lem}

The proof of this lemma is provided in \Cref{subsec:runtime_K_async}.

We provided a comparison of the expected runtimes of $K$-async and $K$-batch-async  SGD variants in \Cref{thm:runtime2}, for the special case of exponential computation times. Here, we provide the proof of \Cref{thm:runtime2}.

\begin{proof}[\textbf{Proof of \Cref{thm:runtime2}}] For the exponential $X_i$, equality holds in \eqref{eqn:T_k_async} in \Cref{lem:runtime kasync}, as we justify in \Cref{subsec:runtime_K_async_exp}.  The expectation can be derived as $\E{X_{K:P}}=\sum_{i=P-K+1}^P \frac{1}{i\mu} \approx \frac{\log{(P/P-K)}}{\mu}$. For exponential $X_i$, the expected runtime per iteration for $K$-batch-async is given by $\E{T}  =  \frac{K\E{X}}{P} =  \frac{K}{\mu P }$ from \Cref{lem:runtime kbatch}.
\end{proof}
 \begin{figure}[t]
\centerline{\includegraphics[height=4cm]{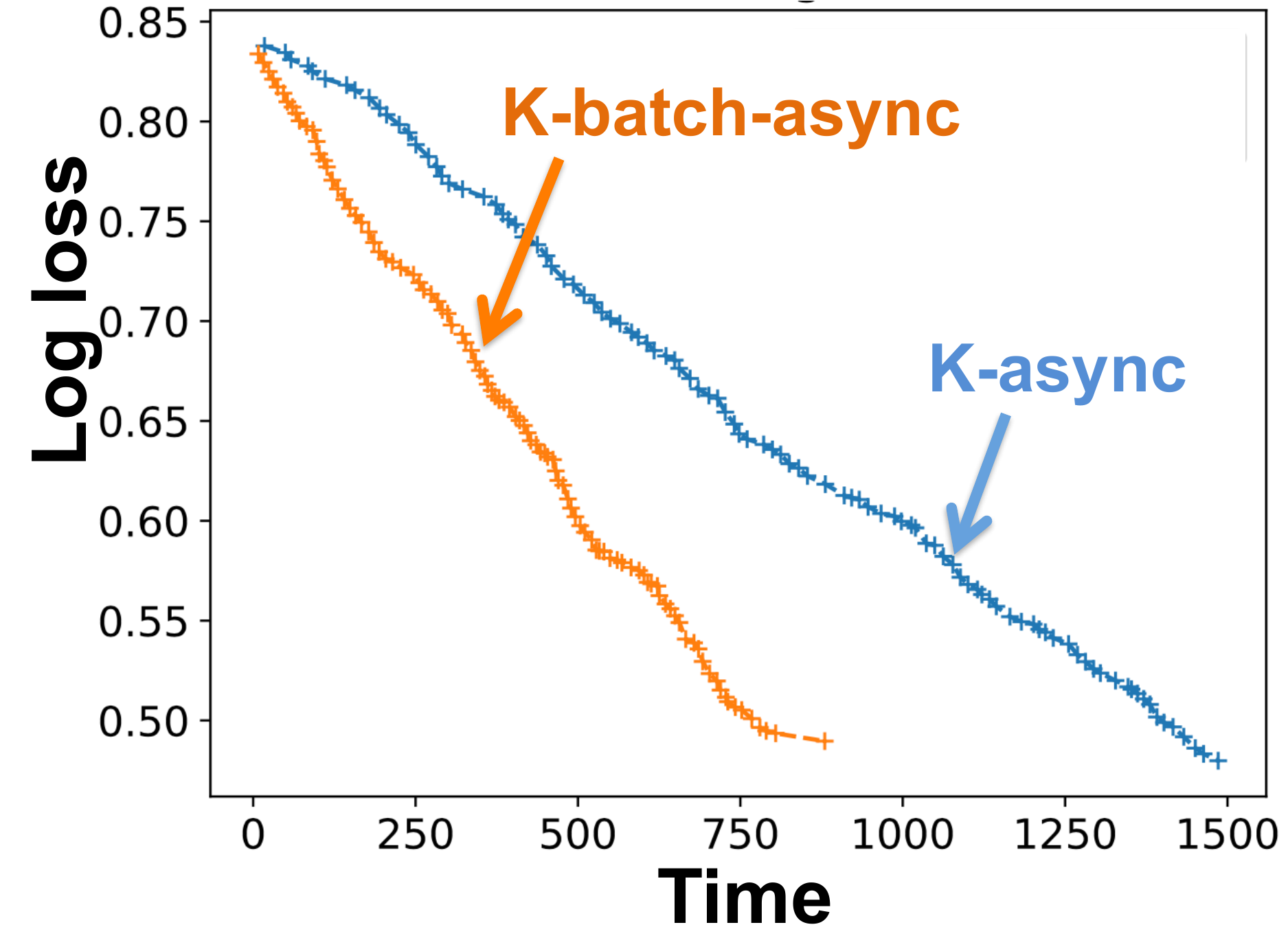}}
\caption{Error-Runtime Trade-off on MNIST Dataset: Comparison of $K$-async with $K$-batch-async under exponential computation time with $X_i \sim \exp(1)$. As derived theoretically, the $K$-batch-async has a sharper fall with time as compared to $K$-async even though the error floor attained is similar. (Details in \Cref{sec:simulation_setup}.)}
\label{fig:async_batch_async}
\end{figure}
\noindent In  \Cref{fig:async_batch_async}, we pictorially illustrate the expected error-runtime trade-offs of $K$-async with $K$-batch-async SGD.






\section{CONCLUSIONS}
\label{sec:conclusion}
The speed of distributed SGD depends on the error reduction per iteration, as well as the runtime per iteration. This paper presents a novel runtime analysis of synchronous and asynchronous SGD, and their variants for any general distribution on the wall-clock time of each learner. When juxtaposed with the error analysis, we get error-runtime trade-offs that can be used to compare different SGD algorithms. We also give a new analysis of asynchronous SGD by relaxing some commonly made assumptions and also propose a novel learning rate schedule to compensate for gradient staleness.

In the future we plan to explore methods to gradually increase synchrony, so that we can achieve fast convergence as well as low error floor. We are also looking into the use of local updates to minimize the frequency of communication between the PS and learners, that is closely related to \cite{zhang2016parallel,yin2017gradient,zhou2017convergence,zhang2015deep}.

\subsection*{Acknowledgements} The authors thank Mark Wegman, Pulkit Grover and Jianyu Wang for their suggestions and feedback.




\bibliographystyle{unsrt}

\bibliography{sample}









\clearpage

\appendix
\section{STRONG CONVEXITY DISCUSSION}
\label{sec:strong_convexity}
\begin{defn}[Strong-Convexity]
A function $h(\mathbf{u})$ is defined to be $c$-strongly convex, if the following holds for all $\mathbf{u}_1$ and $\mathbf{u}_2$ in the domain:
\begin{equation*}
h(\mathbf{u}_2) \geq h(\mathbf{u}_1) + [\nabla h(\mathbf{u}_1)]^T(\mathbf{u}_2-\mathbf{u}_1) + \frac{c}{2} ||\mathbf{u}_2-\mathbf{u}_1 ||_2^2
\end{equation*}
\end{defn}

For strongly convex functions, the following result holds for all $\mathbf{u}$ in the domain of $h(.)$.
\begin{equation}
2c(h(\mathbf{u})- h^*) \leq ||\nabla h(\mathbf{u}) ||_2^2
\end{equation}
The proof is derived in \cite{bottou2016optimization}. For completeness, we give the sketch here.
\begin{proof}
Given a particular $\mathbf{u}$, let us define the quadratic function as follows:
$$q(\mathbf{u}')= h(\mathbf{u}) + \nabla h(\mathbf{u})^T (\mathbf{u}' - \mathbf{u}) + \frac{c}{2} ||\mathbf{u}'-\mathbf{u}  ||_2^2  $$
Now, $q(\mathbf{u}')$ is minimized at $\mathbf{u}' = \mathbf{u} - \frac{1}{c} \nabla h(\mathbf{u}) $ and the value is $h(\mathbf{u}) - \frac{1}{2c} ||\nabla h(\mathbf{u})   ||_2^2$. Thus, from the definition of strong convexity we now have,
\begin{align*}
h^* &\geq h(\mathbf{u}) + \nabla h(\mathbf{u})^T (\mathbf{u}' - \mathbf{u}) + \frac{c}{2} ||\mathbf{u}'-\mathbf{u}  ||_2^2 \nonumber \\
&\geq h(\mathbf{u}) - \frac{1}{2c} ||\nabla h(\mathbf{u})   ||_2^2 \ \ [\text{minimum value of } q(\mathbf{u}')].
\end{align*}
\end{proof}

\section{RUNTIME ANALYSIS PROOFS}
\label{subsec:runtime_K_sync}
Here we provide all the remaining proofs and supplementary information for the results in \Cref{sec:runtime}.

\subsection{Runtime of $K$-sync SGD}
\label{subsec:runtime_K_statistic}
\textbf{$K$-th statistic of exponential distributions:}
Here we give a sketch of why the $K$-th order statistic of $P$ exponentials scales as $\log (P/P-K)$. A detailed derivation can be obtained in \cite{sheldon2002first}. Consider $P$ i.i.d.\ exponential distributions with parameter $\mu$. The minimum $X_{1:P}$ of $P$ independent exponential random variables with parameter $\mu$ is exponential with parameter $P \mu$. Conditional on $X_{1:P}$, the second smallest value $X_{2:P}$ is distributed like the sum of $X_{1:P}$ and an independent exponential random variable with parameter $(P-1)\mu $. And so on, until the $K$-th smallest value $X_{K:P}$ which is distributed like the sum of $X_{(K-1):P}$ and an independent exponential random variable with parameter $(P-K+1)\mu$. Thus, $$X_{K:P}=Y_P+Y_{P-1}+ \dots+Y_{P-K+1} $$
where the random variables $Y_i$s are independent and exponential with parameter $i \mu$. Thus,
$$\E{X_{K:P}} = \sum_{i=P-K+1}^{P} \frac{1}{i\mu} =\frac{H_P-H_{P-K}}{\mu} \approx \frac{\log{\frac{P}{P-K}}}{\mu}. $$ 
Here $H_P$ and $H_{P-K}$ denote the $P$-th and $(P-K)$-th harmonic numbers respectively.

For the case where $K=P$, the expectation is given by,
$$\E{X_{P:P}} = \frac{1}{\mu}\sum_{i=1}^{P} \frac{1}{i} =\frac{1}{\mu}H_P \approx \frac{1}{\mu}\log{P}. $$
 

\subsection{Runtime of $K$-batch-sync SGD}
\label{subsec:runtime_K_batch_sync}

In general, the expected runtime per iteration of $K$-batch-sync SGD is not tractable but for the special case of exponentials it follows the distribution $Erlang(K, P\mu)$. This is obtained from the memoryless property of exponentials. 

All the learners start their computation together. The expected time taken by the first mini-batch to be completed is the minimum of $P$ i.i.d.\ exponential random variables $X_1,X_2,\dots,X_P \sim \exp{(\mu)}$ is another exponential random variable  distributed as $ \exp{(P\mu)}$. At the time when the first mini-batch is complete, from the memoryless property of exponentials, it may be viewed as $P$ i.i.d.\ exponential random variables $X_1,X_2,\dots,X_P \sim \exp{(\mu)}$ starting afresh again. Thus, the time to complete each mini-batch is distributed as $\exp{(P\mu)}$, and an iteration being the sum of the time to complete $K$ such mini-batches, has the distribution $Erlang(K, P\mu)$.

\subsection{Runtime of $K$-batch-async SGD}
\label{subsec:runtime_K_batch_async}
Here we include a discussion on renewal processes for completeness, to provide a background for the proof of \Cref{lem:runtime kbatch}, which gives the expected runtime of $K$-batch-async SGD. The familiar reader can merely skim through this and refer to the proof provided in the main section of the paper in \Cref{sec:runtime}. 

\begin{defn}[Renewal Process]
A renewal process is an arrival process where the inter-arrival intervals are positive, independent and identically distributed random variables.
\end{defn}

\begin{lem}[Elementary Renewal Theorem]
\cite[Chapter 5]{gallager2013stochastic}
\label{lem:renewal}
Let $\{N(t), t>0\} $ be a renewal counting process denoting the number of renewals in time $t$. Let $\E{Z}$ be the mean inter-arrival time. Then,
\begin{equation}
\lim_{t \to \infty} \frac{\E{N(t)}}{t} = \frac{1}{\E{Z}}.
\end{equation}
\end{lem}

Observe that for asynchronous SGD or $K$-batch-async SGD, every gradient push by a learner to the PS can be thought of as an arrival process. The time between two consecutive pushes by a learner follows the distribution of $X_i$ and is independent as computation time has been assumed to be independent across learners and mini-batches. Thus the inter-arrival intervals are positive, independent and identically distributed and hence, the gradient pushes are a renewal process. 

\subsection{Runtime of $K$-async SGD}
\label{subsec:runtime_K_async}
\begin{proof}[Proof of \Cref{lem:runtime kasync}]
For new-longer-than-used distributions observe that the following holds:
\begin{align}
\Pr(X_i > u+t| X_i>t) \leq \Pr(X_i>u).
\end{align}

Thus the random variable $X_i-t | X_i>t$ is thus stochastically dominated by $X_i$. Now let us assume we want to compute the expected computation time of one iteration of $K$-async starting at time instant $t_0$. Let us also assume that the learners last read their parameter values at time instants $t_1,t_2,\dots t_P$ respectively where any $K$ of these $t_1,t_2,\dots t_P$ are equal to $t_0$ as $K$ out of $P$ learners were updated at time $t_0$ and the remaining $(P-K)$ of these $t_1,t_2,\dots t_P$ are $< t_0$. 
Let $Y_1,Y_2,\dots Y_P$ be the random variables denoting the computation time of the $P$ learners starting from time $t_0$. Thus, 
\begin{equation}
Y_i = X_i-(t_0-t_i) | X_i> (t_0-t_i) \ \ \forall \  i=1,2,\dots,P.
\end{equation}
Now each of the $Y_i$ s are independent and are stochastically dominated by $X_i$ s. 

\begin{equation}
\Pr(Y_i >u ) \leq \Pr(X_i >u ) \ \forall \ i,j = 1,2,\dots,P.
\end{equation}
The expectation of the $K$-th statistic of $\{Y_1,Y_2,\dots,Y_P\}$ is the expected runtime of the iteration. 
Let us denote $h_K(x_1,x_2,\dots,x_P)$ as the $K$-th statistic of $P$ numbers  $(x_1,x_2,\dots,x_P)$. And let us us denote $g_{K,\bm{s}}(x)$ as the $K$-th statistic of $P$ numbers where $P-1$ of them are given as $\bm{s}_{1 \times (P-1)}$ and $x$ is the $P-$th number. Thus 
$$g_{K,\bm{s}}(x)=h_K(x, s(1),s(2),\dots,s(P-1)) $$
First observe that $g_{K,\bm{s}}(x)$ is an increasing function of $x$ since given the other $P-1$ values, the $K$-th order statistic will either stay the same or increase with $x$. Now we use the property that if $Y_i$ is stochastically dominated by $X_i$, then for any increasing function $g(.)$, we have
$$\Esub{Y_1}{g(Y_1)} \leq \Esub{X_1}{g(X_1)}.$$
This result is derived in \cite{kreps1990course} .

This implies that for a given $\bm{s}$,
$$ \Esub{Y_1}{g_{K,\bm{s}}(Y_1)} \leq \Esub{X_1}{g_{K,\bm{s}}(X_1)}. $$

This leads to,
\begin{multline}
 \Esub{Y_1|Y_2=s(1),Y_3=s(2) \dots Y_P=s(P-1)}{h_K(Y_1,Y_2,\dots Y_P)} 
\\ \leq \Esub{X_1|Y_2=s(1),Y_3=s(2) \dots Y_P=s(P-1) }{h_K(X_1,Y_2,\dots Y_P)} . 
 \end{multline}

From this, 
\begin{align}
\E{h_K(Y_1,Y_2,\dots Y_P)} 
&=\Esub{Y_2,\dots,Y_P}{ \Esub{Y_1|Y_2, Y_3 \dots Y_P}{h_K(Y_1,Y_2,\dots Y_P)}} \nonumber \\
&\leq \Esub{Y_2,\dots,Y_P}{\Esub{X_1|Y_2,Y_3 \dots Y_P }{h_K(X_1,Y_2,\dots Y_P)} } \nonumber \\
& = \E{h_K(X_1,Y_2,\dots Y_P)}.
 \end{align}

This step proceeds inductively. Thus, similarly
\begin{align}
\E{h_K(X_1,Y_2,\dots Y_P)} 
& =\Esub{X_1,Y_3,\dots,Y_P}{ \Esub{Y_2|X_1, Y_3 \dots Y_P}{h_K(X_1,Y_2,\dots Y_P)}} \nonumber \\
&\leq \Esub{X_1,Y_3,\dots,Y_P}{\Esub{X_2|X_1, Y_3 \dots Y_P }{h_K(X_1,X_2, Y_3,\dots Y_P)} } \nonumber \\
&= \E{h_K(X_1,X_2,Y_3\dots Y_P)}.
 \end{align}

Thus, finally combining, we have,
\begin{align}
\E{h_K(Y_1,Y_2,\dots Y_P)}
&\leq \E{h_K(X_1,Y_2,\dots Y_P)} \nonumber \\
&\leq \E{h_K(X_1,X_2, Y_3\dots Y_P)} \nonumber \\& \leq \dots \nonumber \\
&\leq \E{h_K(X_1,X_2,X_3\dots X_P)}.
 \end{align}
\end{proof}

\subsubsection{Exponential Computation time} 
\label{subsec:runtime_K_async_exp}
For exponential distributions, the inequality in 
\Cref{lem:runtime kasync} holds with equality. This follows from the memoryless property of exponentials. Let us consider the scenario of the proof of \Cref{lem:runtime kasync} where we similarly define $Y_i=X_i-(t_0-t_i)|X_i>(t_0-t_i)$. From the memoryless property of exponentials \cite{sheldon2002first}, if $X_i \sim \exp(\mu)$, then $Y_i \sim \exp(\mu)$. Thus, the expectation of the $K$-th statistic of $Y_i$s can be easily derived as all the $Y_i$s are now i.i.d.\ with distribution $\exp(\mu)$. Thus, the expected runtime per iteration is given by,
$$\E{T}= \E{Y_{K:P}}= \frac{1}{\mu} \sum_{i=P-K+1}^P\frac{1}{i} \approx \frac{1}{\mu} \log{\frac{P}{P-K}} .$$





\section{ASYNC-SGD ANALYSIS PROOFS}
\label{sec:async_proof}

In this section, we provide a proof of the error convergence of asynchronous SGD.
\subsection{Async-SGD with fixed learning rate}
\label{subsec:async_proof}
First we prove a simplified version of \Cref{thm:error kasync} for the case $K=1$. While this is actually a corollary of the more general \Cref{thm:error kasync}, we prove this first for ease of understanding and simplicity. The proof of the more general \Cref{thm:error kasync} is then provided in \Cref{subsec:K_async_proof}. 

The corollary is as follows:
\vspace{0.8cm}
\begin{coro} Suppose that the objective function $F(\wts)$ is strongly convex with parameter $c$ and the learning rate $\eta \leq \frac{1}{2\lips(\frac{M_G}{m}+ 1)}$. Also assume that $\E{|| \nabla F(\wts_{j}) - \nabla F(\wts_{\tau(j)})||_2^2 } \leq \gamma \E{|| \nabla F(\wts_{j}) ||_2^2 }$ for some constant $\gamma \leq 1$. Then, the error after $J$ iterations of Async SGD is given by,
\begin{align*}
 \E{F(\wts_{J})} & -F^* 
\leq \frac{\eta L\sigma^2}{2c\gamma' m} +  (1 - \eta c\gamma')^{J} \left(\E{F(\wts_{0})}-F^* - \frac{\eta L\sigma^2}{2c \gamma' m}  \right),
\end{align*}
where $\gamma'= 1-\gamma + \frac{p_0}{2}$ and $p_0$ is a non-negative lower bound on the conditional probability that $\tau(j)=j$ given all the past delays and parameters.
\label{coro:fixed learning rate}
\end{coro}

To prove the result, we will use the following lemma.
\vspace{0.8cm}
\begin{lem}
\label{lem:bias-variance}
Let us denote $\mathbf{v}_j = g(\wts_{\tau(j)}, \xi_{j})$, and assume that $ \Esub{\xi_{j}|
\wts}{g(\wts,\xi_{j})}= \nabla F(\wts) $. Then,
\begin{align*}
\E{||\nabla F(\wts_{j}) - \mathbf{v}_j||^2_2 } 
& \leq  \E{||\mathbf{v}_j||^2_2} -  
 \E{||\nabla F(\wts_{\tau(j)})||_2^2 } \\
 & \hspace{2cm} +  \E{|| \nabla F(\wts_{j}) - \nabla F(\wts_{\tau(j)})||_2^2 }.
\end{align*}
\end{lem}

\begin{proof}[Proof of \Cref{lem:bias-variance}]
Observe that,
\begin{align}
\E{||\nabla F(\wts_{j}) - \mathbf{v}_j||^2_2 }  &= \E{||\nabla F(\wts_{j}) -\nabla F(\wts_{\tau(j)}) + \nabla F(\wts_{\tau(j)}) - \mathbf{v}_j||^2_2 }  \nonumber \\
&= \E{||\nabla F(\wts_{j}) - \nabla F(\wts_{\tau(j)})||^2_2 } 
 + \E{||\mathbf{v}_j - \nabla F(\wts_{\tau(j)})||^2_2}.
\label{update-bias-variance}
\end{align}

The last line holds since the cross term is $0$ as derived below.
\begin{align*}
&\E{(\nabla F(\wts_{j}) - \nabla F(\wts_{\tau(j)})^T(\mathbf{v}_j - \nabla F(\wts_{\tau(j)}))}  \nonumber \\
& =  \mathbb{E}_{\wts_{\tau(j)},\wts_{j}} [(\nabla F(\wts_{j}) - \nabla F(\wts_{\tau(j)})^T  \Esub{\xi_j|\wts_{\tau(j)},\wts_{j} }{(\mathbf{v}_j - \nabla F(\wts_{\tau(j)}))} ] \nonumber \\
& =  \mathbb{E}_{\wts_{\tau(j)},\wts_{j}} [(\nabla F(\wts_{j}) - \nabla F(\wts_{\tau(j)})^T
(\Esub{\xi_j|\wts_{\tau(j)}}{\mathbf{v}_j} - \nabla F(\wts_{\tau(j)}))] \nonumber \\
&= 0.
\end{align*}
Here again the last line follows from Assumption 2 in \Cref{sec:system model} which states that $$\Esub{\xi_j|\wts_{\tau(j)}}{\mathbf{v}_j} = \nabla F(\wts_{\tau(j)})).$$
Returning to \eqref{update-bias-variance}, observe that the second term can be further decomposed as,
\begin{align*}
 \E{||\mathbf{v}_j - \nabla F(\wts_{\tau(j)})||^2_2} 
& = \Esub{\wts_{\tau(j)}}{ \Esub{\xi_j|\wts_{\tau(j)}}{||\mathbf{v}_j - \nabla F(\wts_{\tau(j)})||^2_2}} \\
&=\Esub{\wts_{\tau(j)}}{ \Esub{\xi_j|\wts_{\tau(j)}}{||\mathbf{v}_j||_2^2}} -2 \Esub{\wts_{\tau(j)}}{ \Esub{\xi_j|\wts_{\tau(j)}}{\mathbf{v}_j^T \nabla F(\wts_{\tau(j)}) }} \\
& \hspace{2cm} + \Esub{\wts_{\tau(j)}}{\Esub{\xi_j|\wts_{\tau(j)}}{|| \nabla F(\wts_{\tau(j)})||_2^2 }}\\
& = \E{||\mathbf{v}_j||_2^2} - 2\E{||\nabla F(\wts_{\tau(j)})  ||_2^2}  + \E{||\nabla F(\wts_{\tau(j)})  ||_2^2 } \\
& = \E{||\mathbf{v}_j||_2^2} - \E{||\nabla F(\wts_{\tau(j)})  ||_2^2 }.
\end{align*}
\end{proof}
We will also be proving a $K$-learner version of this lemma \Cref{subsec:K_async_proof} to prove \Cref{thm:error kasync}. Now we proceed to provide the proof of \Cref{coro:fixed learning rate}.

\begin{proof}[Proof of \Cref{coro:fixed learning rate}]
\begin{align}
 F(\wts_{j+1})   \leq & F(\wts_{j}) + (\wts_{j+1}-\wts_{j} )^T \nabla F(\wts_{j})    + \frac{L}{2} ||\wts_{j+1}-\wts_{j}   ||_2^2 
 \nonumber \\
 = & F(\wts_{j}) + (-\eta \mathbf{v}_j )^T \nabla F(\wts_{j}) + \frac{L \eta^2}{2} || \mathbf{v}_j   ||_2^2 
 \nonumber \\
 = & F(\wts_{j}) - \frac{\eta}{2}||\nabla F(\wts_{j})||_2^2 - \frac{\eta}{2}||\mathbf{v}_j||_2^2 
 + \frac{\eta}{2}||\nabla F(\wts_{j})- \mathbf{v}_j ||_2^2 + \frac{L\eta^2}{2}||\mathbf{v}_j||_2^2. 
 \label{lipschitz condition}
\end{align}

Here the last line follows from $2\bm{a}^T\bm{b} = ||\bm{a}||_2^2 + ||\bm{b}||_2^2 - ||\bm{a}-\bm{b}||_2^2 $. Taking expectation,
\begin{align}
\E{F(\wts_{j+1})} 
 &\leq \E{F(\wts_{j})} - \frac{\eta}{2}\E{||\nabla F(\wts_{j})||_2^2}  - \frac{\eta}{2}\E{||\mathbf{v}_j||_2^2} + \frac{\eta}{2}\E{||\nabla F(\wts_{j})- \mathbf{v}_j ||_2^2}+ \frac{L\eta^2}{2}\E{||\mathbf{v}_j||_2^2} 
 \nonumber\\
& \overset{(a)}{\leq}  \E{F(\wts_{j})} - \frac{\eta}{2}\E{||\nabla F(\wts_{j})||_2^2} - \frac{\eta}{2}\E{||\mathbf{v}_j||_2^2}  
 + \frac{ \eta}{2} \E{|| \mathbf{v}_j||_2^2} - \frac{\eta}{2}\E{||\nabla F(\wts_{\tau(j)})||_2^2} 
\nonumber \\
&\hspace{3cm} +   \frac{\eta}{2}\E{|| \nabla F(\wts_{j}) - \nabla F(\wts_{\tau(j)})||_2^2 } + \frac{L\eta^2}{2}\E{||\mathbf{v}_j||_2^2}. 
\label{termdiff} 
\end{align}
Here, (a) follows from Lemma \ref{lem:bias-variance} that we just derived. Now, again bounding from \eqref{termdiff}, we have
\begin{align}
\E{F(\wts_{j+1})} & \overset{(b)}{\leq} \E{F(\wts_{j})} - \frac{\eta}{2}\E{||\nabla F(\wts_{j})||_2^2}  - \frac{\eta}{2}\E{||\nabla F(\wts_{\tau(j)})||_2^2}  
 + \frac{\eta}{2} \gamma \E{|| \nabla F(\wts_{j})||_2^2 } 
 \nonumber \\
& \hspace{2cm}
 + \frac{L\eta^2}{2}\E{||\mathbf{v}_j||_2^2} \nonumber \\
&\overset{(c)}{\leq} \E{F(\wts_{j})} - \frac{\eta}{2}(1-\gamma)\E{||\nabla F(\wts_{j})||_2^2} + \frac{L\eta^2\sigma^2}{2m}  
\nonumber \\ 
& \hspace{2cm} - \frac{\eta}{2} \left(1 -L\eta(\frac{ M_G}{m}+1)  \right)\E{||\nabla F(\wts_{\tau(j)})||_2^2} 
\nonumber \\
& \overset{(d)}{\leq} \E{F(\wts_{j}) } - \frac{\eta}{2}(1-\gamma)\E{||\nabla F(\wts_{j})||_2^2 } + \frac{L\eta^2\sigma^2}{2m} 
 - \frac{\eta}{4}\E{||\nabla F(\wts_{\tau(j)})||_2^2 } 
\nonumber \\
& \overset{(e)}{\leq} \E{F(\wts_{j}) } - \frac{\eta}{2}(1-\gamma)\E{||\nabla F(\wts_{j})||_2^2 } + \frac{L\eta^2\sigma^2}{2m} 
 - \frac{\eta}{4}p_0\E{||\nabla F(\wts_{j})||_2^2 }. \label{recursion}
\end{align}

Here  (b) follows from the statement of the theorem that $$\E{|| \nabla F(\wts_{j}) - \nabla F(\wts_{\tau(j)})||_2^2 } \leq \gamma \E{|| \nabla F(\wts_{j}) ||_2^2 }$$ for some constant $\gamma \leq 1$.
The next step (c) follows from Assumption 4 in \Cref{sec:system model} which lead to   $$\E{|| \mathbf{v}_j ||_2^2  } \leq \frac{\sigma^2}{m} + \left(\frac{M_G}{m}+1 \right)\E{|| \nabla F(\wts_{\tau(j)}) ||_2^2 }   .$$
Step (d) follows from choosing $\eta < \frac{1}{2L(\frac{M_G}{m} + 1)} $ and finally (e) follows from Lemma \ref{lem:delay}. 

Now one might recall that the function $F(w)$ was defined to be  strongly convex with parameter $c$. 
Using the standard result of strong-convexity \cref{eq:strong-convexity} in \eqref{recursion}, we obtain the following result:
 \begin{align*}
\E{F(\wts_{j+1})}&-F^*  \leq  \frac{\eta^2L\sigma^2}{2m} + (1 - \eta c(1-\gamma + \frac{p_0}{2})) (\E{F(\wts_{j})}-F^*   ).   
\end{align*}
Let us denote $\gamma'= (1-\gamma + \frac{p_0}{2}) $. Then, using the above recursion, we thus have,
\begin{align*}
& \E{F(\wts_{J})} -F^* 
\leq \frac{\eta L\sigma^2}{2c\gamma' m}   + (1 - \eta \gamma'c)^J (\E{F(\wts_{0})}-F^* - \frac{\eta L\sigma^2}{2c\gamma' m}  ).
\end{align*}\vspace{-0.5cm}
\end{proof}

\subsubsection{Discussion on range of $p_0$}
\label{subsec:proof_lem_p_0}
Let us denote the conditional probability of $\tau(j)=j$ given all the past delays and parameters as $p_0^{(j)}$. Now $p_0 \leq p_0^{(j)} \ \forall j$. Clearly the value of $p_0^{(j)}$ will differ for different distributions and accordingly the value of $p_0$ will differ. Here we include a brief discussion on the possible values of $p_0$ for different distributions. These also hold for $K$-async and $K$-batch-async SGD.

\begin{proof}[Proof of \Cref{lem:p_0}]
Let $t_0$ be the time when the $j$-th iteration occurs, and suppose that learner $i'$ pushed its gradient in the $j$-th iteration. Now similar to the proof of \Cref{lem:runtime kasync}, let us also assume that the learners last read their parameter values at time instants $t_1,t_2,\dots t_P$ respectively where $t_i'=t_0$ and the remaining $(P-1)$ of these $t_i$s are $< t_0$. Let $Y_1,Y_2,\dots Y_P$ be the random variables denoting the computation time of the $P$ learners starting from time $t_0$. Thus, $Y_i= X_i -(t_0-t_i)| X_i >(t_0-t_i)$. For exponentials, from the memoryless property, all these $Y_i$ s become i.i.d.\ and thus from symmetry the probability of $i'$ finishing before all the others is equal, \textit{i.e.} $ \frac{1}{P}$. Thus, $p_0^{(j)} =p_0=\frac{1}{P}$.
For new-longer-than-used distributions, as we have discussed before all the $Y_i$s with $i \neq i'$ will be stochastically dominated by $Y_{i'}=X_{i'}$. Thus, probability of $i$s with $i \neq i'$ finishing first is higher than $i'$. Thus, $p_0^{(j)} \leq \frac{1}{P}$ and so is $p_0$. Similarly, for new-shorter-than-used distributions, $Y_{i'}$ is stochastically dominated by all the $Y_i$s and thus probability of $i'$ finishing first is more. So, $p_0^{(j)} \geq \frac{1}{P}$ and so is $p_0$.
\end{proof}


\subsection{K-async SGD under fixed learning rate}
\label{subsec:K_async_proof}
In this subsection, we provide a proof of \Cref{thm:error kasync}.

Before we proceed to the proof of this theorem, we first extend our Assumption 4 from the variance of a single stochastic gradient to sum of stochastic gradients in the following Lemma.

\begin{lem}
\label{lem:assumption4}If the variance of the stochastic updates is bounded as $$\Esub{\xi_{j}|\wts_{\tau{l,j}}}{||g(\wts_{\tau(l,j)},\xi_{l,j})- \nabla F(\wts_{\tau(l,j)}) ||_2^2} \nonumber \\
\leq \frac{\sigma^2}{m} + \frac{M_G}{m} ||\nabla F(\wts_{\tau(l,j)}) ||_2^2 \ \forall \ \tau(l,j) \leq j,$$ then for $K$-async, the variance of the sum of stochastic updates given all the parameter values $\wts_{\tau(l,j)}$ is also bounded as follows:\vspace{-0.5cm}
\begin{align}
&\Esub{\xi_{1,j},\dots,\xi_{K,j}|\wts_{\tau(1,j)} \dots \wts_{\tau(K,j)}}{||\sum_{l=1}^K g(\wts_{l,j}, \xi_{l,j})||_2^2} \nonumber \\ & \leq 
\frac{K\sigma^2}{m} + \left(\frac{M_G}{m}+K\right)||\sum_{l=1}^K \nabla F(\wts_{\tau(l,j)})  ||_2^2.
\end{align}\vspace{-0.6cm}
\end{lem}
\begin{proof}
First let us consider the expectation of any cross term such that $l \neq l'$. For the ease of writing, let $\Omega= \{ \wts_{\tau(1,j)} \dots \wts_{\tau(K,j)} \} $. 

Now observe the conditional expectation of the cross term as follows:\vspace{-0.2cm}
\begin{align}
&\mathbb{E}_{\xi_{1,j},\dots,\xi_{K,j}| \Omega}[(g(\wts_{l,j}, \xi_{l,j})-\nabla F(\wts_{\tau(l,j)}))^T  ((g(\wts_{l',j}, \xi_{l',j})-\nabla F(\wts_{\tau(l',j)}))] \nonumber 
\\
& = \mathbb{E}_{\xi_{l,j},\xi_{l',j}| \Omega}[(g(\wts_{l,j}, \xi_{l,j})-\nabla F(\wts_{\tau(l,j)}))^T ((g(\wts_{l',j}, \xi_{l',j})-\nabla F(\wts_{\tau(l',j)}))] \nonumber 
\\
& = \mathbb{E}_{\xi_{l',j}| \Omega }[\mathbb{E}_{\xi_{l,j}|\xi_{l',j},\Omega}[ (g(\wts_{l,j}, \xi_{l,j})-\nabla F(\wts_{\tau(l,j)}))^T] (g(\wts_{l',j}, \xi_{l',j})-\nabla F(\wts_{\tau(l',j)})] \nonumber 
\\
& = \mathbb{E}_{\xi_{l',j}| \Omega } [0^T  (g(\wts_{l',j}, \xi_{l',j})-\nabla F(\wts_{\tau(l',j)})] 
=0.
\end{align}

Thus the cross terms are all $0$. So the expression simplifies as,\vspace{-0.2cm}
\begin{align}
  \Esub{\xi_{1,j},\dots,\xi_{K,j}|\Omega}{||\sum_{l=1}^K g(\wts_{l,j}, \xi_{l,j}) - F(\wts_{\tau(l,j)})||_2^2} 
& \overset{(a)}{=} \sum_{l=1}^K \Esub{\xi_{1,j},\dots,\xi_{K,j}|\Omega}{||g(\wts_{l,j}, \xi_{l,j})  -F(\wts_{\tau(l,j)})||_2^2} \nonumber \\
& \leq \sum_{l=1}^K \frac{\sigma^2}{m} + \frac{M_G}{m} ||\nabla F(\wts_{\tau(l,j)}) ||_2^2.
\end{align}
Thus,
\begin{align}
&\Esub{\xi_{1,j},\dots,\xi_{K,j}|\Omega}{||\sum_{l=1}^K g(\wts_{l,j}, \xi_{l,j})||_2^2} \nonumber 
\\ 
& = \Esub{\xi_{1,j},\dots,\xi_{K,j}|\Omega}{||\sum_{l=1}^K g(\wts_{l,j}, \xi_{l,j}) - F(\wts_{\tau(l,j)}) ||_2^2} 
+ \Esub{\xi_{1,j},\dots,\xi_{K,j}|\Omega}{||\sum_{l=1}^K  F(\wts_{\tau(l,j)}) ||_2^2} 
\nonumber 
\\
& \leq \frac{K\sigma^2}{m} + \sum_{l=1}^K \frac{M_G}{m}||  F(\wts_{\tau(l,j)}) ||_2^2 + ||\sum_{l=1}^K  F(\wts_{\tau(l,j)}) ||_2^2 \nonumber \\
&\leq \frac{K\sigma^2}{m} + \sum_{l=1}^K \frac{M_G}{m}||  F(\wts_{\tau(l,j)}) ||_2^2 + \sum_{l=1}^K K||  F(\wts_{\tau(l,j)}) ||_2^2.
\end{align}
\end{proof}\vspace{-1cm}

Now we return to the proof of the theorem.

\begin{proof}[Proof of \Cref{thm:error kasync}] Let  $\mathbf{v}_j= \frac{1}{K} \sum_{l=1}^K g(\wts_{l,j}, \xi_{l,j})$. Following steps similar to the Async-SGD proof, from Lipschitz continuity we have the following.
\begin{align}
 F(\wts_{j+1})  & \leq F(\wts_{j}) + (\wts_{j+1}-\wts_{j} )^T \nabla F(\wts_{j})  
  + \frac{L}{2} ||\wts_{j+1}-\wts_{j}   ||_2^2 
 \nonumber \\
 = & F(\wts_{j}) - \frac{\eta}{K} \sum_{l=1}^K g(\wts_{l,j}, \xi_{l,j})^T \nabla F(\wts_{j}) 
+ \frac{L}{2} ||\eta \mathbf{v}_j   ||_2^2  \nonumber 
 \\
 \overset{(a)}{=} & F(\wts_{j}) -\frac{\eta}{2K} \sum_{l=1}^K ||\nabla F(\wts_{j})||_2^2 
- \frac{\eta}{2K} \sum_{l=1}^K ||g(\wts_{l,j}, \xi_{l,j})||_2^2
 \nonumber \\
 &  \hspace{2cm} + \frac{\eta}{2K} \sum_{l=1}^K ||g(\wts_{l,j}, \xi_{l,j}) ||_2^2
 - \frac{\eta}{2K} \sum_{l=1}^K || \nabla F(\wts_{j})  ||_2^2 
 + \frac{L \eta^2}{2} || \mathbf{v}_j   ||_2^2 
 \nonumber \\
 = & F(\wts_{j}) - \frac{\eta}{2}||\nabla F(\wts_{j})||_2^2 
 - \frac{\eta}{2K} \sum_{l=1}^K ||g(\wts_{l,j}, \xi_{l,j})||_2^2
 \nonumber \\
 & \hspace{2cm}
 + \frac{\eta}{2K} \sum_{l=1}^K ||g(\wts_{l,j}, \xi_{l,j}) - \nabla F(\wts_{j})  ||_2^2 
 + \frac{L \eta^2}{2} || \mathbf{v}_j   ||_2^2. 
\end{align}

Here (a) follows from $2\bm{a}^T\bm{b} = ||\bm{a}||_2^2 + ||\bm{b}||_2^2 - ||\bm{a}-\bm{b}||_2^2 $. 
Taking expectation,
\begin{align}
\E{F(\wts_{j+1})} 
 &\leq \E{F(\wts_{j})} - \frac{\eta}{2}\E{||\nabla F(\wts_{j})||_2^2}   - \frac{\eta}{2K} \sum_{l=1}^K \E{||g(\wts_{l,j}, \xi_{l,j})||_2^2} \nonumber\\
 & \hspace{2cm} + \frac{\eta}{2K} \sum_{l=1}^K \E{||\nabla F(\wts_{j})- g(\wts_{l,j}, \xi_{l,j}) ||_2^2}
+ \frac{L\eta^2}{2}\E{||\mathbf{v}_j||_2^2} 
 \nonumber
 \\
& \overset{(a)}{\leq}  \E{F(\wts_{j})} - \frac{\eta}{2}\E{||\nabla F(\wts_{j})||_2^2} 
- \frac{\eta}{2K} \sum_{l=1}^K \E{||g(\wts_{l,j}, \xi_{l,j})||_2^2}
\nonumber \\
&\hspace{2cm}  + \frac{ \eta}{2K}  \sum_{l=1}^K  \E{|| g(\wts_{l,j}, \xi_{l,j})||_2^2}   -\frac{ \eta}{2K}  \sum_{l=1}^K  \E{||\nabla F(\wts_{\tau(l,j)})||_2^2} 
 \nonumber \\
&\hspace{2cm} +   \frac{\eta}{2K}\sum_{l=1}^K \E{|| \nabla F(\wts_{j}) 
 - \nabla F(\wts_{\tau(l,j)})||_2^2 }
 + \frac{L\eta^2}{2}\E{||\mathbf{v}_j||_2^2} 
 \\
& \overset{(b)}{\leq} \E{F(\wts_{j})} - \frac{\eta}{2}\E{||\nabla F(\wts_{j})||_2^2}  - \frac{\eta}{2K}\sum_{l=1}^K \E{||\nabla F(\wts_{\tau(l,j)})||_2^2}  
 + \frac{\eta}{2} \gamma \E{|| \nabla F(\wts_{j})||_2^2 } 
 \nonumber \\
&\hspace{4cm}
+ \frac{L\eta^2}{2}\E{||\mathbf{v}_j||_2^2} \nonumber \\
&\overset{(c)}{\leq} \E{F(\wts_{j})} - \frac{\eta}{2}(1-\gamma)\E{||\nabla F(\wts_{j})||_2^2} + \frac{L\eta^2\sigma^2}{2Km}  
\nonumber \\ 
& \hspace{2cm} - \frac{\eta}{2K} \sum_{l=1}^K \left(1 -L\eta \left(\frac{ M_G}{Km}+\frac{1}{K}\right)  \right)\E{||\nabla F(\wts_{\tau(l,j)})||_2^2} 
\nonumber \\
& \overset{(d)}{\leq} \E{F(\wts_{j}) } - \frac{\eta}{2}(1-\gamma)\E{||\nabla F(\wts_{j})||_2^2 } + \frac{L\eta^2\sigma^2}{2Km} 
 - \frac{\eta}{4K}\sum_{l=1}^K \E{||\nabla F(\wts_{\tau(l,j)})||_2^2 } 
\nonumber \\
& \overset{(e)}{\leq} \E{F(\wts_{j}) } - \frac{\eta}{2}(1-\gamma)\E{||\nabla F(\wts_{j})||_2^2 } + \frac{L\eta^2\sigma^2}{2Km} 
 - \frac{\eta}{4}p_0\E{||\nabla F(\wts_{j})||_2^2 }. 
\label{recursion2}
\end{align}

Here step (a) follows from \Cref{lem:bias-variance} and step (b) follows from the assumption that $$\E{|| \nabla F(\wts_{j}) - \nabla F(\wts_{\tau(l,j)})||_2^2 } \leq \gamma \E{|| \nabla F(\wts_{j}) ||_2^2 }$$ for some constant $\gamma \leq 1$. The next step (c) follows from the \Cref{lem:assumption4} that bounds the variance of the sum of stochastic gradients. Step (d) follows from choosing $\eta < \frac{1}{2L(\frac{M_G}{Km} + \frac{1}{K})} $ and finally (e) follows from \Cref{lem:delay} in \Cref{sec:main results} that says $ \E{||\nabla F(\wts_{\tau(l,j)})||_2^2} \geq p_0\E{||\nabla F(\wts_{j})||_2^2}$ for some non-negative constant $p_0$ which is a lower bound on the conditional probability that $\tau(l,j)=j$ given all past delays and parameter values. 

Finally, since $F(\wts)$ is strongly convex, using the inequality $2c(F(\wts)-F^*) \leq ||\nabla F(\wts)  ||_2^2$ in \eqref{recursion2}, we finally obtain the desired result.
\end{proof}

\subsubsection{Extension to Non-Convex case}
\label{subsec:non_convex}
The analysis can be extended to provide weaker guarantees for non-convex objectives. Let $\gamma'=1-\gamma + \frac{p_0}{2} $

For non-convex objectives, we have the following result.
\begin{thm} For non-convex objective function, we have the following ergodic convergence result given by:
$$\frac{1}{J+1} \sum_{j=0}^J \E{|| \nabla F(\wts_j)   ||_2^2} \leq \frac{2(F(\wts_0)-F^*)}{(J+1) \eta \gamma' } +  \frac{L\eta\sigma^2}{Km\gamma'}$$
where $F^*=\min_{\wts} F(\wts)$.
\end{thm}

\begin{proof}
Recall the recursion derived in the last proof in \eqref{recursion2}. After re-arrangement, we obtain the following:
\begin{align}
\E{||\nabla F(\wts_j)   ||_2^2} & \leq \frac{2(\E{F(\wts_{j})}-\E{F(\wts_{j+1}}))}{\eta \gamma'} + \frac{L\eta\sigma^2}{Km\gamma'}.
\end{align}

Taking summation from $j=0$ to $j=J$, we get,
\begin{align}
\frac{1}{J+1}  \sum_{j=0}^J \E{|| \nabla F(\wts_j)   ||_2^2}    & \leq\frac{2(\E{F(\wts_0)}-\E{F(\wts_J)})}{(J+1) \eta \gamma' } +  \frac{L\eta\sigma^2}{Km\gamma'} \nonumber \\
& \overset{(a)}{\leq}\frac{2(F(\wts_0)-F^*)}{(J+1) \eta \gamma' } +  \frac{L\eta\sigma^2}{Km\gamma'}.
\end{align}
Here (a) follows since we assume $\wts_0$ to be known and also from   $\E{F(\wts_J)} \geq F^*$.
\end{proof}


\subsection{Variable Learning Rate Schedule} 
\label{subsec:variable learning rate}
We propose a new heuristic for learning rate schedule that is more stable than fixed learning rate for asynchronous SGD. Our learning rate schedule is $
\eta_j = \min\left\{   \frac{C}{||\wts_j-\wts_{\tau(j)} ||_2^2}, \eta_{max}
\right\}
$, where $\eta_{max}$ is a suitably large value of learning rate beyond which the convergence diverges. This heuristic is inspired from the assumption in Theorem \ref{thm:variable learning rate} given by $\eta_j\E{ ||\wts_j -\wts_{\tau(j)} ||_2^2} \leq C$. In this section, we derive the accuracy trade-off mentioned in Theorem \ref{thm:variable learning rate} based on this assumption.

\begin{proof}[Proof of Theorem \ref{thm:variable learning rate}] Following steps similar to \eqref{lipschitz condition}, we first obtain the following:
\begin{align}
 F(\wts_{j+1})  
 &\leq  F(\wts_{j}) - \frac{\eta_j}{2}||\nabla F(\wts_{j})||_2^2 - \frac{\eta_j}{2}||\mathbf{v}_j||_2^2 
 \nonumber \\
& \hspace{2cm} + \frac{\eta_j}{2}||\nabla F(\wts_{j})- \mathbf{v}_t ||_2^2 + \frac{L\eta_j^2}{2}||\mathbf{v}_j||_2^2. 
\end{align}
Now taking expectation, we obtain the following result.
\begin{align}
\E{F(\wts_{j+1})}  
& 
\overset{(a)}{\leq}  \E{F(\wts_{j})} - \frac{\eta_j}{2}\E{||\nabla F(\wts_{j})||_2^2} - \frac{\eta_j}{2}\E{||\mathbf{v}_j||_2^2}  
 + \frac{ \eta_j}{2} \E{|| \mathbf{v}_j||_2^2}
 \nonumber \\
&\hspace{1.8cm}
- \frac{\eta_j}{2}\E{||\nabla F(\wts_{\tau(j)})||_2^2} 
 +   \frac{\eta_j}{2}\E{|| \nabla F(\wts_{j}) - \nabla F(\wts_{\tau(j)})||_2^2 }
 + \frac{L\eta_j^2}{2}\E{||\mathbf{v}_j||_2^2}
\nonumber \\
& 
\overset{(b)}{\leq} 
\E{F(\wts_{j})} - \frac{\eta_j}{2}\E{||\nabla F(\wts_{j})||_2^2}  
  - \frac{\eta_j}{2}\E{||\nabla F(\wts_{\tau(j)})||_2^2} + \frac{CL^2}{2}  + \frac{L\eta_j^2}{2}\E{||\mathbf{v}_j||_2^2} \nonumber \\
& 
\overset{(c)}{\leq}
\E{F(\wts_{j})} - \frac{\eta_j}{2}\E{||\nabla F(\wts_{j})||_2^2} +  \frac{CL^2}{2}+ \frac{L\eta_j^2\sigma^2}{2m} 
\nonumber \\
&
\hspace{1.8 cm}- \frac{\eta_j}{2} \left(1-L\eta_j(\frac{ M_G}{m}+1) \right)\E{||\nabla F(\wts_{\tau(j)})||_2^2} 
\nonumber \\
& \overset{(e)}{\leq}
\E{F(\wts_{j}) } - \frac{\eta_j}{2}\E{||\nabla F(\wts_{j})||_2^2 } 
  +  \frac{CL^2}{2}+ \frac{\eta_j^2L\sigma^2}{2m} - \frac{\eta_j}{4}\E{||\nabla F(\wts_{\tau(j)})||_2^2 }.
\end{align}
Here (a) follows from \cref{termdiff}, (b) follows from \cref{eq:stalenessbound}, (c) follows from Assumption 4 and (d) follows as $\eta_j \leq \frac{1}{2L(\frac{M_G}{m}+1 )}$. Let us define $\Delta_j= \frac{CL^2}{2}+ \frac{\eta_j^2L\sigma^2}{2m}  $. Thus, the recursion can be written as,
\begin{align}
\E{F(\wts_{j+1})}  &\leq \E{F(\wts_{j}) } - \frac{\eta_j}{2}\E{||\nabla F(\wts_{j})||_2^2 } 
- \frac{\eta_j}{4}\E{||\nabla F(\wts_{\tau(j)})||_2^2 } +  \Delta_j 
\nonumber \\
&\overset{(e)}{\leq}\E{F(\wts_{j})} - \frac{\eta_j}{2}(1+ \frac{p_0}{2})\E{||\nabla F(\wts_{j})||_2^2 }  + \Delta_j.
\end{align} 
Here (e) follows from \Cref{lem:delay}.
If the loss function $F(\wts)$ is strongly convex with parameter $c$, then for all $\wts$, we have 
$2c(F(\wts)-F^* ) \leq ||\nabla F(\wts)||_2^2$.
Using this result, we obtain
\begin{align}
 \E{F(\wts_{j+1})} -F^*  &\leq (1- \eta_j(1+ \frac{p_0}{2})c ) (\E{F(\wts_{j})} -F^*) 
 + \Delta_j 
\nonumber 
\\
& \leq (1- \eta_j(1+ \frac{p_0}{2})c )(1- \eta_{j-1}(1+ \frac{p_0}{2})c ) (\E{F(\wts_{j-1})} -F^*) \nonumber  \\  & \hspace{1cm} + (1- \eta_j(1+ \frac{p_0}{2})c )\Delta_{j-1} +  \Delta_j  \nonumber \\
& \leq (1-\rho_j)(1-\rho_{j-1})\dots(1-\rho_0)(\E{F(\wts_{0})}- F^*) + \Delta,
\end{align}
where $\rho_j=\eta_j(1+ \frac{p_0}{2})c$ and  $\Delta= \Delta_j + (1-\rho_j)\Delta_{j-1}+  \dots + (1-\rho_j)(1-\rho_{j-1})\dots(1-\rho_1)\Delta_{0} $.
\end{proof}

\section{SIMULATION SETUP DETAILS}
\label{sec:simulation_setup}
MNIST \cite{lecun1998mnist}: For the simulations on MNIST dataset, we first convert the $28\times 28$ images into single vectors of length $784$. We use a single layer of neurons followed by soft-max cross entropy with logits loss function. Thus effectively the parameters consist of a weight matrix $\bm{W}$ of size $784 \times 10$ and a bias vector $\bm{b}$ of size $1\times 10$. We use a regularizer of value $0.01$, mini-batch size $m=1$, and learning rate $\eta=0.01$. For implementation we used Tensorflow with Python3. Thus, the model is as follows:
\begin{verbatim}
X=tf.placeholder(tf.float32,[None,784])
Y=tf.placeholder(tf.float32,[None,10])
W=tf.Variable(tf.random_normal(shape=[784,10],
               stddev=0.01), name="weights")
b=tf.Variable(tf.random_normal(shape=[1,10],
               stddev=0.01),  name="bias")

logits=tf.matmul(X,W) + b
entropy=tf.nn.softmax_cross_entropy_with
         _logits(logits=logits,labels=Y) +
                   lamda*tf.square(tf.norm(W))

loss=tf.reduce_mean( entropy)
\end{verbatim}

For the run-time simulations, we generate random variables from the respective distributions in python to represent the computation times.

CIFAR10 \cite{krizhevsky2009learning}: For the CIFAR10 simulations,  similar to MNIST, we convert the images into vectors of length $1024$. We combine the three colour variants in the ratio $ [0.2989,0.5870,0.114] $ to generate a single vector of length $1024$ for every image. We use a single layer of neurons again followed by soft-max cross entropy with logits in tensorflow. Thus, the parameters consist of a weight matrix $\bm{W}$ of size $ 1024 \times 10$ and a bias vector $ \bm{b} $ of size $ 1 \times 10 $. We use a mini-batch size of $250$, regularizer of $0.05$.

We use a similar model as follows:
\begin{verbatim}
X=tf.placeholder(tf.float32,[None,1024])
Y=tf.placeholder(tf.float32,[None,10])
W=tf.Variable(tf.random_normal(shape=[1024,10],
            stddev= 0.01),name="weights")
b=tf.Variable(tf.random_normal(shape=[1,10],
            stddev = 0.01),name="bias")
            
logits=tf.matmul(X,W) +  b
entropy=tf.nn.softmax_cross_entropy_with
           _logits(logits=logits,labels=Y) +
           lamda*tf.square(tf.norm(W))
loss=tf.reduce_mean(entropy)
\end{verbatim}

The computation time as each learner is generated from exponential distribution.


\section{CHOICE OF HYPERPARAMETERS}
\label{sec:hyperparameters}
Our analysis techniques can also inform the choice of hyperparameters for synchronous and $K$-sync SGD.

\subsection{Varying $K$ in $K$-sync}
We first perform some simulations of $K$-sync SGD applied on the MNIST dataset. For the simulation set-up, we consider $8$ parallel learners with fixed mini-batch size $m=1$ and fixed learning rate $0.05$. The number of learners to wait for in $K$-sync, \textit{i.e.} $K$ is varied and the error-runtime trade-off is observed. The runtimes are generated from a shifted exponential distribution given by $X_i \sim m + \exp{\mu}$.

\begin{figure}[t]
\centerline{\includegraphics[height=4cm]{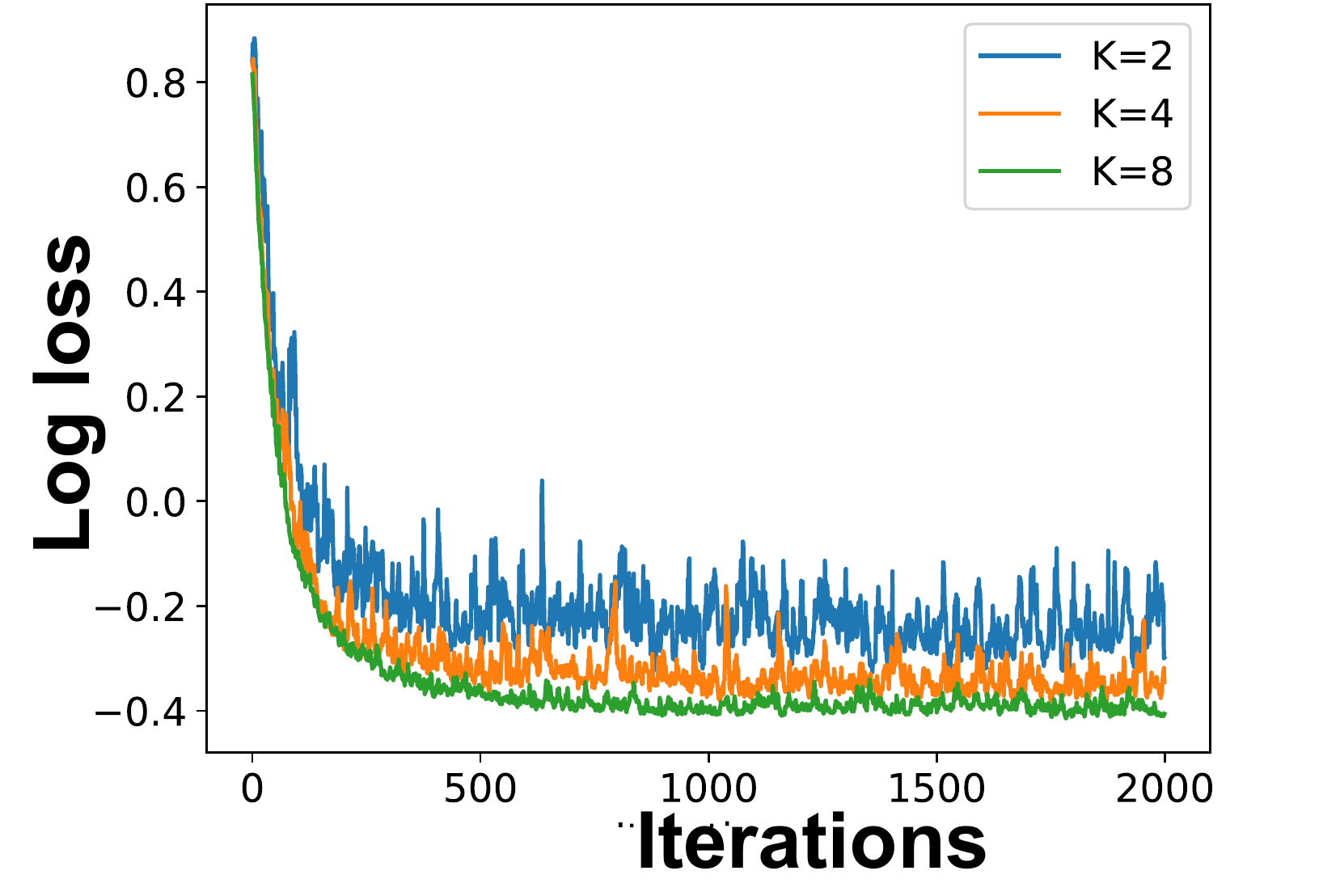}}
\caption{Error-Iterations tradeoff on MNIST dataset: Simulation of $K$-sync SGD for different values of $K$. Observe that accuracy improves with increasing $K$ which  means increasing effective batch size ($\eta=0.05$).  }
\label{fig:hyperparameter1}
\end{figure}

\begin{figure}[t]
\centerline{\includegraphics[height=4cm]{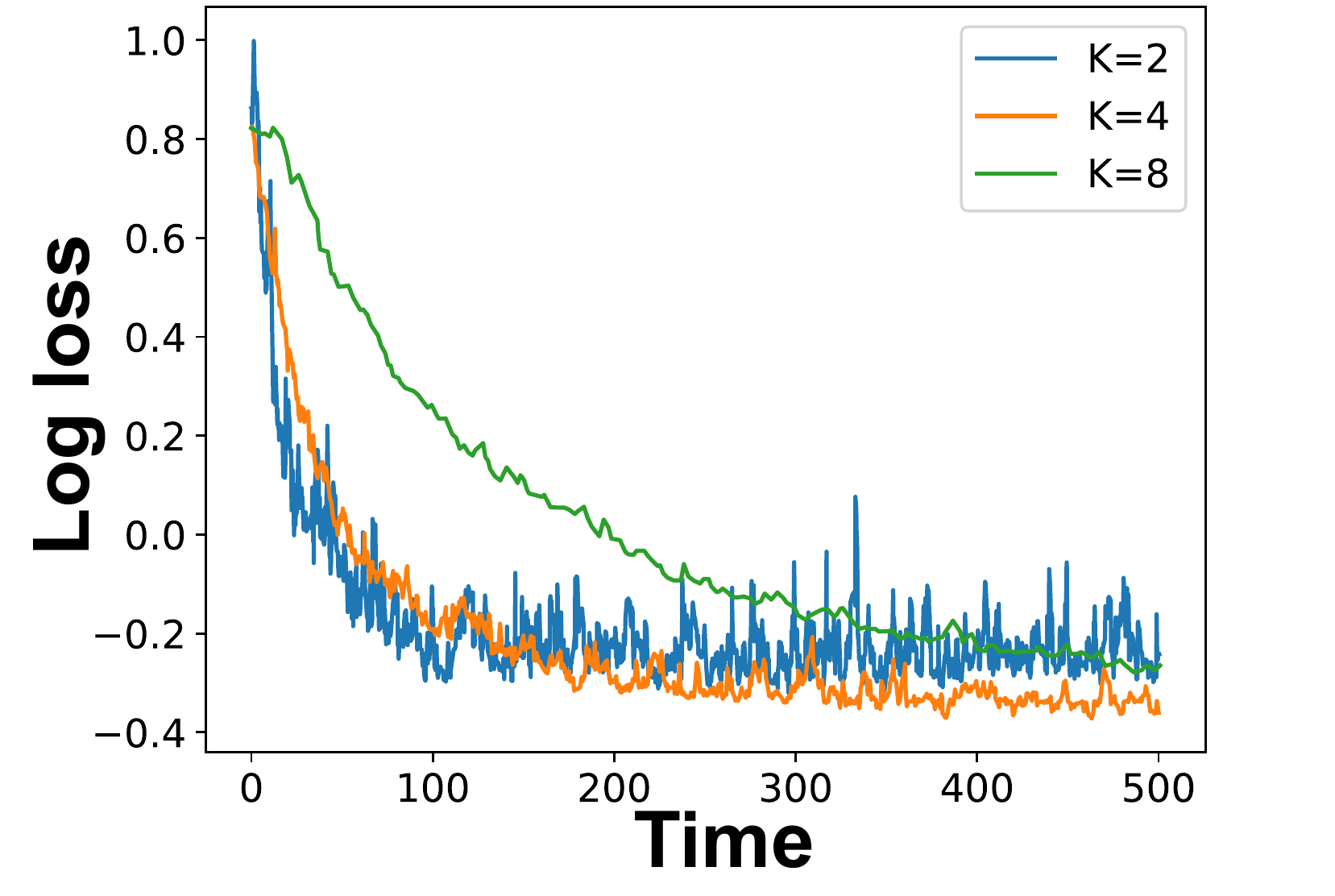}}
\caption{Error-Runtime tradeoff on MNIST dataset: Simulation of $K$-sync SGD for different values of $K$ ($\eta=0.05$).}\vspace{-0.3cm}
\label{fig:hyperparameter2}
\end{figure}

Observe that in the plot of error with the number of iterations in \Cref{fig:hyperparameter1}, the error improves with increasing $K$, which means increasing the effective mini-batch and reducing the variability in the gradient. However, if we look at the same error plotted against runtime (See \Cref{fig:hyperparameter2}) instead of the number of iterations, observe that increasing $K$ naively does not always lead to a better trade-off. As $K$ increases, the central PS has to wait for more learners to finish at every iteration, thus suffering from increased straggler effect. The best error-runtime trade-off is obtained at an intermediate $K=4$. Thus, the current analysis informs the optimal choice of $K$ to achieve a good error-runtime trade-off.

\subsection{Varying mini-batch $m$ }
We consider the training of Alexnet on ImageNet dataset \cite{krizhevsky2012imagenet} using $P=4$ learners. For this simulation, we perform fully synchronous SGD, \textit{i.e.} $K$-sync with $K=P=4$. We fix the learning rate and vary the mini-batch used for training. The runtimes are generated from a shifted exponential distribution given by $X_i \sim m + \exp{\mu}$, that depends on the mini-batch size. Intuitively, this distribution makes sense since to compute one mini-batch, a processor would atleast need a time $m$ (Work Complexity). However, due to delays, it has the additional exponential tail. The error-runtime trade-offs are observed in \Cref{fig:hyperparameter3} and \Cref{fig:hyperparameter4}.

\begin{figure}[ht]
\centerline{\includegraphics[height=4cm]{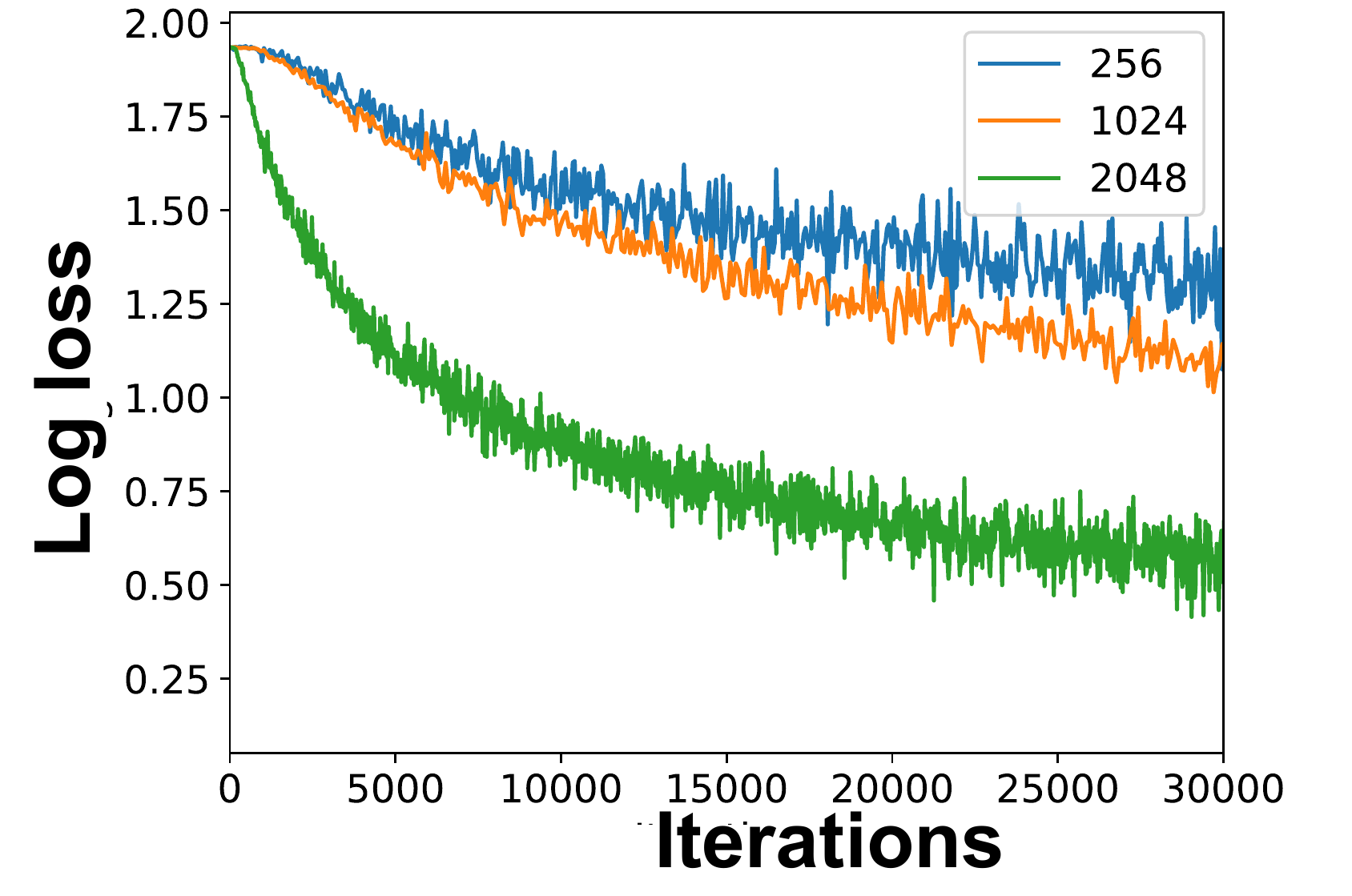}}
\caption{Error-Iterations tradeoff on IMAGENET dataset: Simulation of fully synchronous SGD ($K=P=4$) for different values of mini-batch $m$. Observe that accuracy improves with increasing $m$ which means increasing effective batch size.}
\label{fig:hyperparameter3}
\end{figure}

\begin{figure}[ht]
\centerline{\includegraphics[height=4cm]{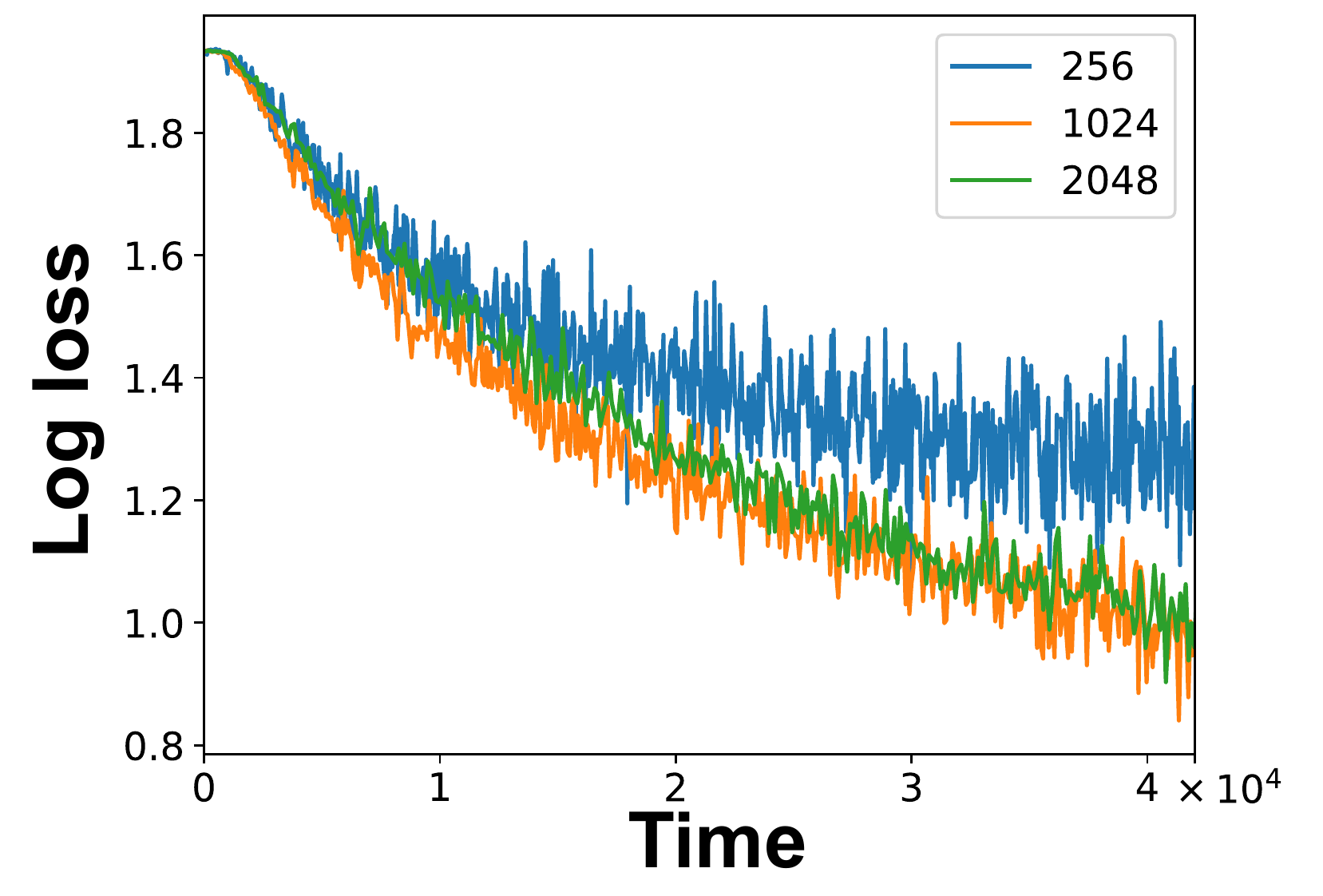}}
\caption{Error-Runtime tradeoff on IMAGENET dataset: Same simulation of fully synchronous SGD ($K=P=4$) for different values of mini-batch $m$ plotted against time. Observe that higher $m$ does not necessarily mean the best trade-off with runtime as higher mini-batch also has longer time.}
\label{fig:hyperparameter4}
\end{figure}

Again, observe that the plot of error with the number of iterations improves with the mini-batch size, as also expected from theory. However, increasing the mini-batch also changes the runtime distribution. Thus, when we plot the same error against expected runtime, we again observe that increasing the mini-batch size naively does not necessarily lead to the best trade-off. Instead, the best error-runtime trade-off is observed with an intermediate mini-batch value of $1024$. Thus, our analysis informs the choice of the optimal mini-batch.
\end{document}